\pdfoutput=1

\documentclass[10pt,journal,compsoc]{IEEEtran}
\usepackage{times}
\usepackage{helvet}
\usepackage{courier}
\usepackage{algorithm}
\usepackage{algorithmic}
\usepackage{amsthm}
\usepackage{amsmath}
\usepackage[mathscr]{euscript}
\usepackage{amssymb}
\usepackage{bm}
\usepackage{multirow}
\usepackage{graphics}
\usepackage{graphicx}
\usepackage{caption}
\usepackage{subcaption}
\usepackage{hhline}
\newtheorem{theorem}{Theorem}
\usepackage{tikz}
\newtheorem{lemma}[theorem]{Lemma}
\DeclareMathOperator{\Tr}{Tr}
\DeclareMathOperator*{\argmin}{arg\,min}

\newcommand{\matr}[1]{\bm{#1}}
\DeclareRobustCommand*\circled[1]{\tikz[baseline=(char.base)]{\node[shape=circle,draw,inner sep=0.5pt] (char) {#1};}}

\ifCLASSOPTIONcompsoc
\usepackage[nocompress]{cite}
\else
\usepackage{cite}
\fi
\ifCLASSINFOpdf
\else
\fi
\begin{document}
	%
	\title{Relative Pairwise Relationship Constrained \\ Non-negative Matrix Factorisation}
	%
	%
	%
	%
	
	\author{Shuai Jiang, 
		Kan Li, 
		and Richard Yida Xu
		\IEEEcompsocitemizethanks{
			\IEEEcompsocthanksitem S. Jiang is with the School of Computer Science, Beijing Institute of Technology (BIT), Beijing, China, and the Faculty of Engineering and Information Technology, University of Technology Sydney (UTS), Australia (email: jiangshuai@bit.edu.cn; shuai.jiang-1@student.uts.edu.au).
			\IEEEcompsocthanksitem K. Li is with the School of Computer Science, Beijing Institute of Technology (BIT), Beijing, China (email: likan@bit.edu.cn).
			\IEEEcompsocthanksitem R. Y. D. Xu is with the Faculty of Engineering and Information Technology, University of Technology Sydney (UTS), Australia (email: yida.xu@uts.edu.au).
	}}
	
	%
	%

	\markboth{IEEE TRANSACTIONS ON KNOWLEDGE AND DATA ENGINEERING}%
	{Shell \MakeLowercase{\textit{et al.}}: Bare Demo of IEEEtran.cls for Computer Society Journals}
	%



	\IEEEtitleabstractindextext{%
		\begin{abstract}
			Non-negative Matrix Factorisation (NMF) has been extensively used in machine learning and data analytics applications. Most existing variations of NMF only consider how each row/column vector of factorised matrices should be shaped, and ignore the relationship among pairwise rows or columns. In many cases, such pairwise relationship enables better factorisation, for example, image clustering and recommender systems. In this paper, we propose an algorithm named, Relative Pairwise Relationship constrained Non-negative Matrix Factorisation (RPR-NMF), which places constraints over relative pairwise distances amongst features by imposing penalties in a triplet form. Two distance measures, squared Euclidean distance and Symmetric divergence, are used, and exponential and hinge loss penalties are adopted for the two measures respectively. It is well known that the so-called ``multiplicative update rules'' result in a much faster convergence than gradient descend for matrix factorisation. However, applying such update rules to RPR-NMF and also proving its convergence is not straightforward. Thus, we use reasonable approximations to relax the complexity brought by the penalties, which are practically verified. Experiments on both synthetic datasets and real datasets demonstrate that our algorithms have advantages on gaining close approximation, satisfying a high proportion of expected constraints, and achieving superior performance compared with other algorithms.
		\end{abstract}
		
		}

	\maketitle

	\IEEEdisplaynontitleabstractindextext

	%
	\IEEEpeerreviewmaketitle

	\IEEEraisesectionheading{\section{Introduction}\label{sec:introduction}}

	%
	%
	%
	%
	\IEEEPARstart{C}{ompared} to conventional dimensionality reduction methods, such as Singular Value Decomposition (SVD), low rank Non-negative Matrix Factorisation (NMF), mostly solving an optimisation task, converges much faster when it comes down to large real-world data sets \cite{LeeNature1999,BrunetPNAS2004,DingPAMI2010}. Thus NMF has been widely used in many applications \cite{YangWSDM2013,MohammadihaTASLP2013}, and algorithms of this kind have been the research foci in many communities, such as image processing and recommender systems \cite{KorenC2009,EsserTIP2012,MohammadihaTASLP2013,LiuSDM2013,KimSIGKDD2015}.
	
	A seminal approach in NMF is the so-called ``multiplicative update rules'' which guarantees both the convergence of the algorithm and the non-negativity of factorised matrices \cite{LeeNIPS2001}. Though the ``multiplicative update rules'' were proved not converging to a stationary point numerically \cite{Gonzalez2005DCAM}, and they are not strictly well-defined because of possible zero entries \cite{LinNC2007}, it practically produces satisfactory results, especially for large scale data, which makes it a popular solution for NMF. However, the original NMF only imposes the non-negativity constraints on both of the factorising and factorised matrices, which in practice may not be enough to satisfy additional requirements. Thus, researchers in this area have been proposing new algorithms under this framework to cater for incremental improvements, variations, and/or application oriented constraints \cite{HoyerJMLR2004,DingSDM2005,PascualPAMI2006,OzerovTASLP2010,SandlerPAMI2011,KimuraML2016}.
	
	A main sub category of NMF is Constrained Non-negative Matrix Factorisation (CNMF), which imposes constraints based on variables as regularisation terms \cite{PaucaLAA2006}. The most commonly used regularisations for NMF are L1 norm and L2 norm, the former increases the sparseness of the factorised matrices, while the latter makes the results smooth to prevent overfitting. However, these constraints are only imposed on each of the rows or columns and has not considered the relationship among pairwise rows or columns. Such pairwise relationships exists wildly in many systems, especially in those systems where the factorised matrices are representing features. A typical instance is the matrix factorisation technique used in recommender systems.
	
	In a recommender system, the factorising matrix is usually the rating matrix whose entries denote the ratings given by the corresponding user (row) to the corresponding item (column), and the factorised matrices are usually regarded as the user feature matrix (the left factorised matrix) and the item feature matrix (the right factorised matrix). Fig. \ref{movie} shows an example of a simple movie recommender system. In this example, after factorisation by NMF, the distance between ``Star Wars'' and ``Titanic'' becomes less than the distance between ``Star Wars'' and ``Star Trek''. However, as we all know, the movie ``Star Wars'' should be closer to the movie ``Star Trek'' -- both are within the scientific fiction genre, unlike ``Titanic'', which is a love story. Thus intuitively, it will result in a better factorisation and make better recommendations if we can incorporate such human-aware relative relationships into the NMF model.
	
	\begin{figure*}[htbp]
		\begin{center}
			\includegraphics[width=1\textwidth]{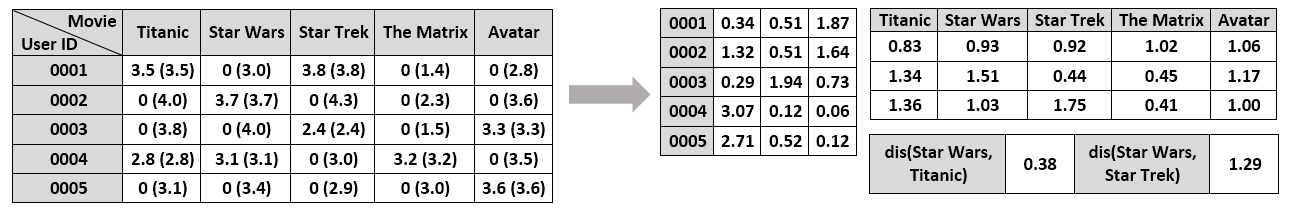}
			\caption{\small{A simple movie recommender system. The factorisation algorithm used here is NMF with Euclidean measure proposed in \cite{LeeNIPS2001}. Original missing ratings are denoted by $0$ and recovered ratings are showed in brackets. As shown under the right factorised matrix, the distance between feature vectors of movies ``Star Wars'' and ``Titanic'' is less than the distance between movies ``Star Wars'' and ``Star Trek''.}\label{movie}}
		\end{center}
	\end{figure*}
	
	There have been a few studies that attempt to consider the above pairwise relationship in the area of matrix factorisation. Although none of these have properly addressed our problem, two existing methods are still worth noting here: one is Graph Regularised Non-negative Matrix Factorisation (GNMF) \cite{CaiPAMI2011}, the other is Label Constrained Non-negative Matrix Factorisation (LCNMF) \cite{LiuPAMI2012}.
	
	GNMF constructs a weight matrix of the graph from the observed data, and then applies the weights (similarities) on the factorised low-dimensional data representation as regularisations. It is designed as a dimensionality reduction method, and it works well on image clustering applications where the data points in different classes are distinctively different. However, in many other cases, such as where the data points are not spread and where the matrix factorisation is not used to reduce dimensionality (like in recommender systems), GNMF cannot guarantee the relative relationship denoted by similarities retained as expected after factorisation. Besides, the setting of similarities is sensitive to the factorisation results, especially when the similarities cannot be simply set zeros and ones, such as when there exists chain constraints.
	
	LCNMF was proposed to cater for scenarios where partially labeled grouping data was made available. Its key idea is, if two feature vectors are labeled into the same class, they are assumed to have the same feature representation in the latent space. This approach has addressed the need of applications on image clustering, however, such a setting is far too restrictive in general: ``Star Wars'' and ``Star Trek'' could be very similar to each other, but setting their features identical is unacceptable and impractical.
	
	In this paper, we propose a novel matrix factorisation algorithm, called RPR-NMF. Rather than using explicit similarities or previously known labels, RPR-NMF imposes penalties for relative pairwise relationships (RPRs) in a triplet form. The penalties are not limited to be within $[0,1]$ as for similarities or to be binary values as for labels. Both of the squared Euclidean distance and the symmetric divergence measure are used in the objective of RPR-NMF, and the penalties are in exponential and hinge loss forms respectively. The update rules for RPR-NMF conform to the well-known ``multiplicative update rules'' in which the proofs of convergence are essential for the whole algorithm. Due to the complexity of proof brought by the imposed penalties, we approximate partial terms in the proof part and have verified its practical benefits through numerous experiments. 
	
	Compared with the existing methods GNMF and LCNMF, RPR-NMF can guarantee more pairwise relationships retrained after factorisation. Fig. \ref{eg} gives a demonstration of the RPRs among four data points after running GNMF, LCNMF, and our proposed algorithm RPR-NMF using Euclidean measure respectively. In this example, GNMF failed on retaining one RPR (points \circled{$2$} and \circled{$4$} should be closer than points \circled{$3$} and \circled{$4$}), LCNMF projected all points onto one, while RPR-NMF retains all RPRs after factorisation.
	
	\begin{figure*}[htbp]
		\begin{center}
			\includegraphics[width=0.96\textwidth]{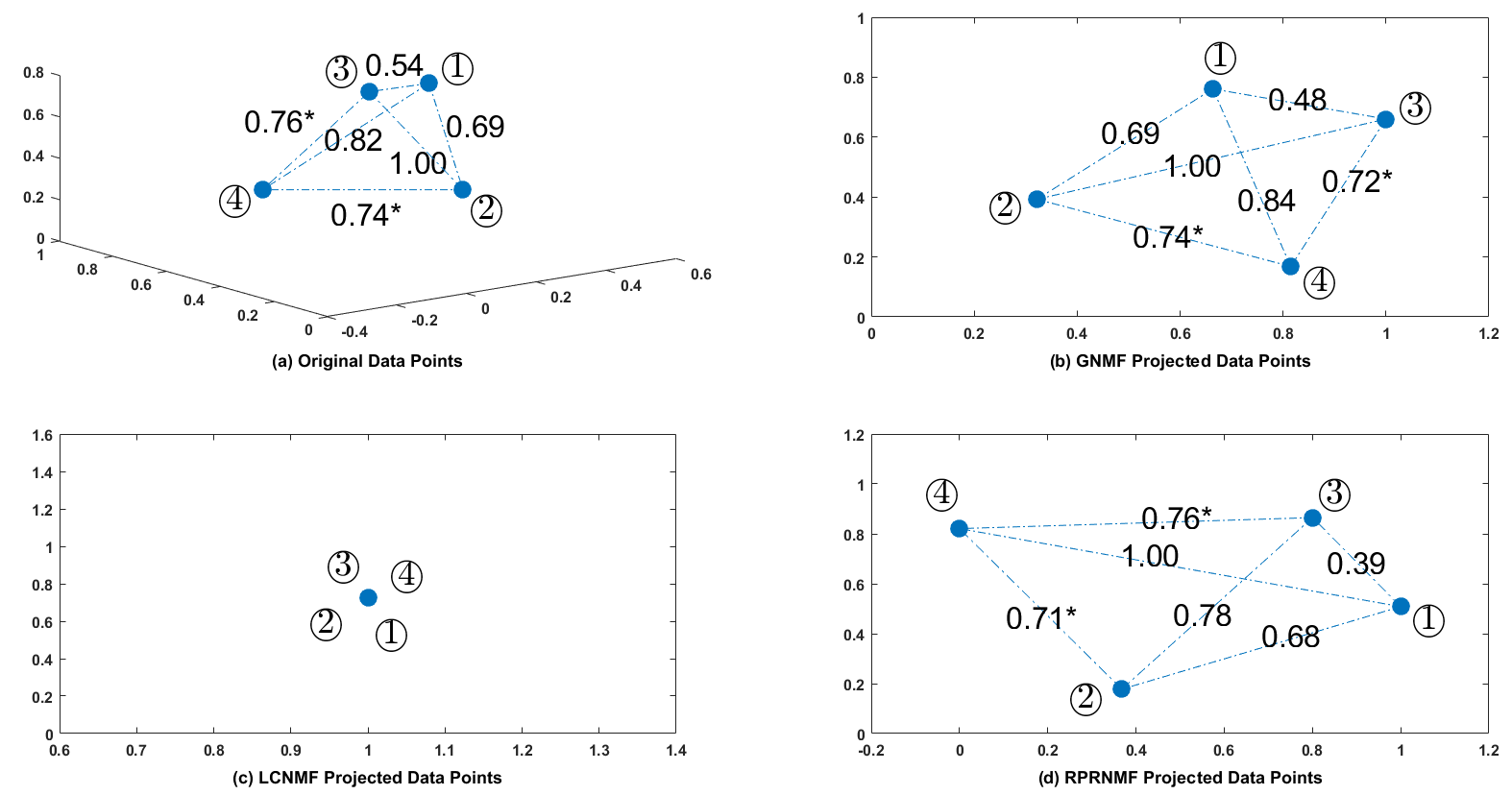}
			\caption{An example of projecting four 3D data points into 2D space. Points are ordered by circled numbers, and the value near each dotted line is the normalised Euclidean distance. (a) The original data space. The distance between point \circled{$2$} and point \circled{$4$} is less than the distance between point \circled{$3$} and point \circled{$4$} (marked by asterisk sign). (b) Projected data points by GNMF using Euclidean distance between each pair of original data points as dissimilarity matrix. The distance between point \circled{$2$} and point \circled{$4$} becomes bigger than the distance between point \circled{$3$} and point \circled{$4$}. (c) LCNMF projects all four points onto one when considering all RPRs. (d) RPR-NMF retains all the RPRs after factorisation.}\label{eg}
		\end{center}
	\end{figure*}
	
	The main contributions of this paper are:
	
	1. We propose a novel algorithm named RPR-NMF, which utilises relative pairwise relationship among rows or columns of factorised matrices to achieve a better factorisation with high constraint satisfied rate and close approximation simultaneously. Different from GNMF and LCNMF, RPR-NMF can guarantee the expected RPRs retained after factorisation and does not limit the feature vectors to be identical.
	
	2. In the part of method for RPR-NMF, we use different forms of penalties for Euclidean measure and Divergence measure based on our observations on method convergence through numerous experiments. For the Euclidean measure, we incorporate the RPRs in natural exponential functions, while for the Divergence measure, we use hinge loss function.
	
	3. The solution of RPR-NMF conforms to that of ``multiplicative rules'', in which we provide complete and sufficient proofs for both of the distance measures with the help of relaxation on partial terms. Such relaxation has been verified reasonable and practical through our experiments.
	
	4. The complexity analysis shows that RPR-NMF does not increase much of processing time by introducing penalty terms when comparing to NMF. Synthetic and real datasets experiments both demonstrate that RPR-NMF have advantages on close approximation, high constraint satisfied rate and outstanding application performance.
	
	The rest of this paper is organised as follows: in the next section, we review related literature to our work. The Method section contains the details of our algorithm and the proof of convergence, followed by the Experiments section in which we evaluate our algorithm as well as the Conclusions.
	
	\section{Related Work}
	
	NMF was first proposed to solve the following optimisation problem: given a non-negative matrix $\matr{V}$, find non-negative matrix factors $\matr{W}$ and $\matr{H}$ such that $\matr{V} \approx \matr{WH}$. Since the seminal work of \cite{LeeNIPS2001} which has proposed the so-called, ``multiplicative update rules'' for non-negative matrix factorisation, a number of approaches followed suit dealing with various NMF issues from different aspects. Our work falls under the category of the Constrained Non-negative Matrix Factorisation (CNMF) which was first proposed in \cite{PaucaLAA2006}. In general, it has the following representation: 
	\begin{equation}
	\label{cnmf}
	\min_{\matr{W},\matr{H}}\{ \lVert \matr{V} - \matr{W}\matr{H} \rVert^{2}_{F} + \alpha J_{1}(\matr{W}) + \beta J_{2}(\matr{H})\},
	\end{equation}
	for $\matr{W} \geqslant 0$ and $\matr{H} \geqslant 0$, where $\lVert \cdot \rVert_{F}$ is the Frobenius norm, $\alpha$ and $\beta$ are regularisation coefficients. The functions of $J_{1}(\matr{W})$ and $J_{2}(\matr{H})$ are penalty terms used to enforce certain constraints on the solution of Eq.(\ref{cnmf}). In their work, the penalties are set $J_{1}(\matr{W}) = \lVert \matr{W} \rVert_{F}^{2}$ and $J_{2}(\matr{H}) = \lVert \matr{H} \rVert_{F}^{2}$ in order to enforce the smoothness in $\matr{W}$ and $\matr{H}$ respectively.
	
	Many studies followed the above CNMF framework. \cite{JiaGRS2009} used an adaptive potential function as penalties to characterise the piecewise smoothness of spectral data. \cite{LiCVPR2001} imposed three additional constraints on the NMF basis to reveal local features. \cite{ShenAAAI2010} imposed both L1 and L2 norms to control the shape of base matrix and increase the sparseness of the coefficient matrix so as to enhance the clustering performance on multiple manifolds.
	
	Besides the above studies based on CNMF that consider how to constrain the value of each vector in factorised matrices, there are two algorithms that take the relationship among vectors of factorised matrix into consideration: Graph Regularised Non-negative Matrix Factorisation (GNMF) proposed by \cite{CaiPAMI2011}, and Labelled Constrained Non-negative Matrix Factorisation (LCNMF) proposed by \cite{LiuPAMI2012}.
	
	GNMF, also a CNMF algorithm, is to utilise the relationship among factorised rows or columns for a dimensionality reduction issue. It is in an effort to ensure that similarities between data points at the original space are also retained after they are transformed to the low dimensional subspace through factorisation. Its objective function with Euclidean measure is as following:
	\begin{equation}
	\mathscr{F} = \lVert \matr{V} - \matr{WH} \rVert^{2}_{F} + \lambda \Tr(\matr{H}\matr{L}\matr{H}^{T}),
	\end{equation}
	where $\matr{L}$ is a graph Laplacian matrix obtained from the similarity matrix and $\lambda$ is a regularisation coefficient. As showed in the objective function, GNMF tries to minimise two parts simultaneously: the squared errors between the product of factorised matrices and the factorising matrix, and the similarity matrix. The ideal solution is to minimise them at the same time. However, it often happens that if GNMF minimised the squared error, the RPRs implied by similarities might not be satisfied. As shown in Fig. \ref{eg}, the RPR among points \circled{$2$}, \circled{$3$} and \circled{$4$} was opposite to that when they were in the original 3D space. Our proposed RPR-NMF works in a different way: it imposes penalties with respect to the expected RPRs, which forces the factorised feature vectors to keep as many of the expected RPRs as possible.
	
	LCNMF also utilises the relationship among factorised vectors. Instead of imposing regularised constraints, it represents the relationship by altering the factorising structure. Thus it is not a CNMF method. LCNMF uses partial label information as hard constraints and turns the original NMF task into a semi-supervised problem: they represent the right factorised matrix by a product of a class matrix and a reduced feature matrix where the class matrix contains binary entries to divide data into a predefined number of classes. Its objective function with Euclidean measure is as:
	\begin{equation}
	\mathscr{F} = \lVert \matr{V} - \matr{W}\matr{Y}\matr{B} \rVert^{2}_{F},
	\end{equation}
	where $\matr{Y}$ is the reduced feature matrix and $\matr{B}$ is the class matrix. The method however, assumes that if two data points have the same label, their corresponding feature vectors must be identical, as showed in Fig. \ref{eg} where all four data points are labelled the same and are projected onto one point after factorisation. Such constraints are too restrictive under many general settings. For example, if two movies are by the same director, in the same genre and even feature the same actors, setting their features identical is to ignore any difference between them, which is what our method aims to mitigate.
	
	\section{Method}
	
	In this section, we introduce a new factorisation algorithm when RPR constraints are in place, called RPR-NMF. Note that both GNMF and LCNMF only impose constraints on the right factorised matrix because they were proposed as data dimensionality reduction methods. For generality, RPR-NMF imposes constraints on both factorised matrices, and it is trivial to only impose constraints on one factorised matrix.
	
	Consider a dataset represented by a non-negative $N \times M$ matrix $\matr{V}$. This matrix is then approximately factorised into an $N \times K$ matrix $\matr{W}$ and a $K \times M$ matrix $\matr{H}$, where $K$ is usually set to be smaller than both $N$ and $M$, and is commonly referred to as the latent dimension. The RPR constraints placed on the factorised matrices can be defined as two sets of integer indexed triples:
	\begin{eqnarray}
	\matr{L}_{\matr{W}} \subseteq \{(q,r,s) | q,r,s \in \mathbb{N}^{+}, q,r,s \leqslant n, q \neq r \neq s\}, \\
	\matr{L}_{\matr{H}} \subseteq \{(q,r,s) | q,r,s \in \mathbb{N}^{+}, q,r,s \leqslant m, q \neq r \neq s\}.
	\end{eqnarray}
	
	Specifically, each triple represents the relative relationship between two pairs of vectors with one sharing vector. In our work, if $\matr{W}$ was the matrix in question and $l^{\text{th}}$ triple specified the distance between vector $q$ and $r$ to be less than the distance between vectors $q$ and $s$, the relationship could be denoted as $dis(\matr{W}_{q^{l}:},\matr{W}_{r^{l}:}) < dis(\matr{W}_{q^{l}:},\matr{W}_{s^{l}:})$ where $\matr{W}_{q:}$ is the $q^{\text{th}}$ row vector of matrix $\matr{W}$, and $dis(\bm{x},\bm{y})$ measures the distance between vectors $\bm{x}$ and $\bm{y}$. We follow the most commonly used two distance measures in our work, which are the squared Euclidean distance
	\begin{equation}
	E(\bm{x},\bm{y}) = \lVert \bm{x} - \bm{y} \rVert^{2},
	\end{equation}
	and the Divergence (when the variables are two unit vectors/distributions, it becomes KL-Divergence)
	\begin{equation}
	D(\bm{x}||\bm{y}) = \sum_{i = 1}^{K} \bm{x}_{i} \log\frac{\bm{x}_{i}}{\bm{y}_{i}} - \bm{x}_{i} + \bm{y}_{i}.
	\end{equation}
	Since the Divergence of two vectors is not symmetric ($D(\bm{x}||\bm{y}) \neq D(\bm{y}||\bm{x})$), when characterising the imposed RPR constraints, we use the Symmetric Divergence defined as
	\begin{equation}
	\begin{aligned}
	SD(\bm{x}, \bm{y}) & = \frac{1}{2}(D(\bm{x} || \bm{y}) + D(\bm{y} || \bm{x})) \\
	& = \frac{1}{2}\sum_{i=1}^{K} (\bm{x}_{i} - \bm{y}_{i})\log\frac{\bm{x}_{i}}{\bm{y}_{i}}.
	\end{aligned}
	\end{equation}
	
	Then the constraints are incorporated as penalty terms in the objective function. In our work, the penalty format for Euclidean measure is of an addition of natural exponential functions, while that for Divergence measure is of a hinge loss function. The reason for not using the same format of penalties is that the Divergence measure with exponential penalties cannot guarantee a high proportion of satisfied constrains after factorisation, and that the Euclidean measure with hinge loss penalties cannot steadily converge. Thus with two independent coefficients $\lambda_{\matr{W}}$ and $\lambda_{\matr{H}}$, we define the objective function using Euclidean distance as
	\begin{equation}
	\label{obj1}
	\begin{aligned}
	& \mathscr{F}_{1} = \lVert \matr{V} - \matr{W}\matr{H} \rVert_{F}^{2} \\
	&\; + \lambda_{\matr{W}} \sum_{l = 1}^{l_{\matr{W}}} [\exp( E(\matr{W}_{q^{l}:},\matr{W}_{r^{l}:})) + \exp(-E(\matr{W}_{q^{l}:}, \matr{W}_{s^{l}:}))] \\
	&\; + \lambda_{\matr{H}} \sum_{l = 1}^{l_{\matr{H}}} [\exp( E(\matr{H}_{:q^{l}},\matr{H}_{:r^{l}})) + \exp(-E(\matr{H}_{:q^{l}}, \matr{H}_{:s^{l}}))]\\
	&\; s.t. \lambda_{\matr{W}} \geqslant 0, \lambda_{\matr{H}} \geqslant 0, \forall i,j, \matr{W}_{ij} \geqslant 0, \matr{H}_{ij} \geqslant 0,
	\end{aligned}
	\end{equation}
	and the objective function using Divergence as
	\begin{equation}
	\label{obj2}
	\begin{aligned}
	& \mathscr{F}_{2} = D(\matr{V} || \matr{W}\matr{H})\\
	&\; + \lambda_{\matr{W}} \sum_{l = 1}^{l_{\matr{W}}} \max(0, SD(\matr{W}_{q^{l}:},\matr{W}_{r^{l}:}) - SD(\matr{W}_{q^{l}:}, \matr{W}_{s^{l}:})) \\
	&\; + \lambda_{\matr{H}} \sum_{l = 1}^{l_{\matr{H}}} \max(0, SD(\matr{H}_{:q^{l}},\matr{H}_{:r^{l}}) - SD(\matr{H}_{:q^{l}}, \matr{H}_{:s^{l}})) \\
	&\; s.t. \lambda_{\matr{W}} \geqslant 0, \lambda_{\matr{H}} \geqslant 0, \forall i,j, \matr{W}_{ij} \geqslant 0, \matr{H}_{ij} \geqslant 0,
	\end{aligned}
	\end{equation}
	where $l_{\matr{W}}$ and $l_{\matr{H}}$ are the numbers of constraints.
	
	\subsection{Solving Objective Functions}
	To solve the above objective functions, we need to derive the update rules for $\matr{W}$ and $\matr{H}$. As a matter of fact, the penalties are not convex even when fixing one of the matrix factors. However, we found it is still feasible to obtain the update rules by constructing and solving the corresponding Lagrange functions. Once we obtained the updating rules, we could minimise the objective functions by iteratively updating $\matr{W}$ and $\matr{H}$.
	
	\subsubsection{Updating rules for Euclidean measure}
	As for the objective function in Eq.(\ref{obj1}), we first construct a Lagrange function with non-negative constraints $\matr{W}_{ij} \geqslant 0$ and $\matr{H}_{ij} \geqslant 0$:
	\begin{equation}
	\begin{aligned}
	& \mathscr{L}_{1} = \lVert \matr{V} - \matr{W}\matr{H} \rVert_{F}^{2} + \sum_{ij} \alpha_{ij}\matr{W}_{ij} + \sum_{ij} \beta_{ij}\matr{H}_{ij} \\\
	&\; + \lambda_{\matr{W}} \sum_{l = 1}^{l_{\matr{W}}} [\exp( E(\matr{W}_{q^{l}:},\matr{W}_{r^{l}:})) + \exp(-E(\matr{W}_{q^{l}:}, \matr{W}_{s^{l}:}))] \\
	&\; + \lambda_{\matr{H}} \sum_{l = 1}^{l_{\matr{H}}} [\exp( E(\matr{H}_{:q^{l}},\matr{H}_{:r^{l}})) + \exp(-E(\matr{H}_{:q^{l}}, \matr{H}_{:s^{l}}))].
	\end{aligned}
	\end{equation}
	The partial derivative with respect to $\matr{W}_{ab}$ is:
	\begin{equation}
	\begin{aligned}
	\frac{\partial\mathscr{L}_{1}}{\partial\matr{W}_{ab}} &= 2[-(\matr{V}\matr{H}^{T})_{ab} + (\matr{W}\matr{H}\matr{H}^{T})_{ab} \\
	&\quad + \lambda_{\matr{W}} C_{row}(\matr{W}_{ab})] + \alpha_{ab},
	\end{aligned}
	\end{equation}
	where
	\begin{equation}
	\label{CW}
	\begin{small}
	\begin{aligned}
	& C_{row}(\matr{W}_{ab}) = \sum_{l = 1}^{l_{\matr{W}}} \bigg\{\\
	& \exp(E(\matr{W}_{q^{l}:},\matr{W}_{r^{l}:}))[ \sum_{q^{l} = a} (\matr{W}_{q^{l}b} - \matr{W}_{r^{l}b}) + \sum_{r^{l} = a} (\matr{W}_{r^{l}b} - \matr{W}_{q^{l}b})] -\\
	& \exp(-E(\matr{W}_{q^{l}:}, \matr{W}_{s^{l}:}))[\sum_{q^{l} = a} (\matr{W}_{q^{l}b} - \matr{W}_{s^{l}b}) + \sum_{s^{l} = a} (\matr{W}_{s^{l}b} - \matr{W}_{q^{l}b})] \bigg\}\\
	& = \sum_{l = 1}^{l_{\matr{W}}}\bigg(\exp(E(\matr{W}_{q^{l}:},\matr{W}_{r^{l}:})) (\sum_{q^{l} = a} \matr{W}_{q^{l}b} + \sum_{r^{l} = a} \matr{W}_{r^{l}b}) + \\
	& \exp(-E(\matr{W}_{q^{l}:}, \matr{W}_{s^{l}:}))(\sum_{q^{l} = a} \matr{W}_{s^{l}b} + \sum_{s^{l} = a} \matr{W}_{q^{l}b})\bigg) - \\
	& \sum_{l = 1}^{l_{\matr{W}}}\bigg(\exp(E(\matr{W}_{q^{l}:},\matr{W}_{r^{l}:}))(\sum_{q^{l} = a} \matr{W}_{r^{l}b} + \sum_{r^{l} = a} \matr{W}_{q^{l}b}) + \\
	& \exp(-E(\matr{W}_{q^{l}:}, \matr{W}_{s^{l}:}))(\sum_{q^{l} = a} \matr{W}_{q^{l}b} + \sum_{s^{l} = a} \matr{W}_{s^{l}b})\bigg) \\
	& = C_{row}^{+}(\matr{W}_{ab}) - C_{row}^{-}(\matr{W}_{ab}),
	\end{aligned}
	\end{small}
	\end{equation}
	
	Let the partial derivative vanish and considering the non-negative constraints ($\alpha_{ab}\matr{W}_{ab} = 0$ under K.K.T conditions), we obtain:
	\begin{equation}
	\begin{aligned}
	((\matr{V}\matr{H}^{T})_{ab} + \lambda_{\matr{W}}& C_{row}^{-}(\matr{W}_{ab}))\matr{W}_{ab} -\\
	((\matr{W}\matr{H}\matr{H}^{T})_{ab} & + \lambda_{\matr{W}} C_{row}^{+}(\matr{W}_{ab}))\matr{W}_{ab} = 0,
	\end{aligned}
	\end{equation}
	thus the update rule for $\matr{W}_{ab}$ is formulated as following:
	\begin{equation}
	\label{updateW1}
	\matr{W}_{ab} \leftarrow \matr{W}_{ab}\frac{(\matr{V}\matr{H}^{T})_{ab} + \lambda_{\matr{W}} C_{row}^{-}(\matr{W}_{ab})}{(\matr{W}\matr{H}\matr{H}^{T})_{ab} + \lambda_{\matr{W}} C_{row}^{+}(\matr{W}_{ab})}.
	\end{equation}
	Similarly, we have the update rule for $\matr{H}_{ab}$ as
	\begin{equation}
	\label{updateH1}
	\matr{H}_{ab} \leftarrow \matr{H}_{ab}\frac{(\matr{W}^{T}\matr{V})_{ab} + \lambda_{\matr{H}} C_{col}^{-}(\matr{H}_{ab})}{(\matr{W}^{T}\matr{W}\matr{H})_{ab} + \lambda_{\matr{H}} C_{col}^{+}(\matr{H}_{ab})}.
	\end{equation}
	
	As for the update rules, we have the following theorem:
	\begin{theorem}
		\label{theorem1}
		The objective function $\mathscr{F}_{1}$ in Eq.(\ref{obj1}) is non-increasing under the update rules in Eq.(\ref{updateW1}) and Eq.(\ref{updateH1}) with appropriate penalty coefficients $\lambda_{\matr{W}}$ and $\lambda_{\matr{H}}$.
	\end{theorem}
	
	We provide the prove of the convergence for the above updating rules and Theorem \ref{theorem1} in Section \ref{proof1}.
	
	\subsubsection{Updating rules for Divergence measure}
	As for the objective function in Eq.(\ref{obj2}), we first construct a Lagrange function with non-negative constraints $\matr{W}_{ij} \geqslant 0$ and $\matr{H}_{ij} \geqslant 0$:
	\begin{equation}
	\begin{aligned}
	& \mathscr{L}_{2} = D(\matr{V} || \matr{WH}) + \sum_{ij} \alpha_{ij}\matr{W}_{ij} + \sum_{ij} \beta_{ij}\matr{H}_{ij} \\
	&\; + \lambda_{\matr{W}} \sum_{l = 1}^{l_{\matr{W}}} \max(0, SD(\matr{W}_{q^{l}:},\matr{W}_{r^{l}:}) - SD(\matr{W}_{q^{l}:}, \matr{W}_{s^{l}:})) \\
	&\; + \lambda_{\matr{H}} \sum_{l = 1}^{l_{\matr{H}}} \max(0, SD(\matr{H}_{:q^{l}},\matr{H}_{:r^{l}}) - SD(\matr{H}_{:q^{l}}, \matr{H}_{:s^{l}})).
	\end{aligned}
	\end{equation}
	
	The partial derivative with respect to $\matr{W}_{ab}$ is:
	\begin{equation}
	\begin{aligned}
	\frac{\partial\mathscr{L}_{2}}{\partial\matr{W}_{ab}} &= \sum_{j}(\matr{H}_{bj} - \frac{\matr{V}_{aj}\matr{H}_{bj}}{(\matr{WH})_{aj}}) \\
	&\quad + \frac{1}{2}\lambda_{\matr{W}} P_{row}(\matr{W}_{ab}) + \alpha_{ab},
	\end{aligned}
	\end{equation}
	where
	\begin{equation}
	\label{PW}
	\begin{aligned}
	& P_{row}(\matr{W}_{ab}) = \sum_{l = 1}^{l_{\matr{W}}} \bigg\{\sum_{q^{l} = a} [g(\matr{W}_{q^{l}b}^{+},\matr{W}_{r^{l}b}^{+}) + g(\matr{W}_{q^{l}b}^{+},\matr{W}_{s^{l}b}^{+})] \\
	& + \sum_{r^{l}=a} g(\matr{W}_{r^{l}b}^{+},\matr{W}_{q^{l}b}^{+}) - \sum_{s^{l} = a} g(\matr{W}_{s^{l}b}^{+},\matr{W}_{q^{l}b}^{+})\bigg\},
	\end{aligned}
	\end{equation}
	\begin{equation}
	\matr{W}_{q^{l}b}^{+} = 
	\begin{cases}
	0, & \text{if the constraint $l$ is satisfied} \\
	\matr{W}_{q^{l}b} & \text{otherwise} 
	\end{cases}
	\end{equation}
	\begin{equation}
	g(x,y) = log\frac{x}{y} + (x - y)\frac{1}{x}.
	\end{equation}
	
	Let the partial derivative equal to zero as well as considering the non-negative constraints ($\alpha_{ab}\matr{W}_{ab} = 0$ under K.K.T conditions), we obtain:
	\begin{equation}
	- \matr{W}_{ab}\sum_{j}\frac{\matr{V}_{aj}\matr{H}_{bj}}{(\matr{WH})_{aj}} + \matr{W}_{ab}(\frac{1}{2}\lambda_{\matr{W}}P_{row}(\matr{W}_{ab}) + \sum_{j}\matr{H}_{bj}) = 0,
	\end{equation}
	thus the update rule for $\matr{W}_{ab}$ is formulated as following:
	\begin{equation}
	\label{updateW2}
	\matr{W}_{ab} \leftarrow \matr{W}_{ab}\frac{\sum_{j}\matr{V}_{aj}\matr{H}_{bj}/(\matr{WH})_{aj}}{\frac{1}{2}\lambda_{\matr{W}}P_{row}(\matr{W}_{ab}) + \sum_{j}\matr{H}_{bj}}.
	\end{equation}
	Similarly, we have the update rule for $\matr{H}_{ab}$ as
	\begin{equation}
	\label{updateH2}
	\matr{H}_{ab} \leftarrow \matr{H}_{ab}\frac{\sum_{i}\matr{V}_{ib}\matr{W}_{ia}/(\matr{WH})_{ib}}{\frac{1}{2}\lambda_{\matr{H}}P_{col}(\matr{H}_{ab}) + \sum_{i}\matr{W}_{ia}}.
	\end{equation}
	Notice that the value of $P_{row}(\matr{W}_{ab})$ and $P_{col}(\matr{H}_{ab})$ may be negative during updating. Thus if the denominator in an iteration is less than zero, then we abandon the penalty parts. Besides, we dynamically change the penalty coefficients to ensure the convergence since the update rules are obtained by approximation (see proof of convergence in Section \ref{proof2}).
	
	As for the update rules, we have the following theorem:
	\begin{theorem}
		\label{theorem2}
		The objective function $\mathscr{F}_{2}$ in Eq.(\ref{obj2}) is non-increasing under the update rules in Eq.(\ref{updateW2}) and Eq.(\ref{updateH2}) with appropriate penalty coefficients $\lambda_{\matr{W}}$ and $\lambda_{\matr{H}}$.
	\end{theorem}
	
	We provide the prove of the convergence for the above updating rules and Theorem \ref{theorem2} in Section \ref{proof2}.
	
	\subsection{Proofs of Convergence and Theorems} 
	We construct an auxiliary function $\mathscr{G}(x, x^{\prime})$ to help prove Theorem \ref{theorem1} \& \ref{theorem2}, which satisfies the conditions $\mathscr{G}(x, x^{\prime}) \geqslant \mathscr{F}(x)$ and $\mathscr{G}(x,x) = \mathscr{F}(x)$, and guarantees $\mathscr{F}(x)$ to be non-increasing under the following update:
	\begin{equation}
	\label{arg}
	x^{t+1} = \argmin_x \mathscr{G}(x,x^{\prime}).
	\end{equation}
	\quad We illustrate our proofs for $\matr{W}_{ab}$ and the proofs for $\matr{H}_{ab}$ can be derived in a similar fashion. Let $\mathscr{F}_{ab}(\matr{W})$ denote part of $\mathscr{F}(\matr{W})$ concerning $\matr{W}_{ab}$ and so as $\mathscr{G}_{ab}(\matr{W},\matr{W}^{t})$.
	
	\subsubsection{Convergence of Euclidean updating rules and Theorem 1}
	\label{proof1}
	As for updating rules in Eq.(\ref{updateW1}) \& (\ref{updateH1}), we have
	\begin{lemma}
		Function
		\begin{equation}
		\label{G1}
		\begin{aligned}
		\mathscr{G}_{ab}&(\matr{W},\matr{W}^{t}) = \mathscr{F}_{ab}(\matr{W}^{t}) + (\matr{W}_{ab} - \matr{W}_{ab}^{t})\mathscr{F}_{ab}^{\prime}(\matr{W}^{t}) \\
		+ & (\matr{W}_{ab} - \matr{W}_{ab}^{t})^{2} \frac{(\matr{W}^{t}\matr{H}\matr{H}^{T})_{ab}  + \lambda_{\matr{W}}C_{row}^{+}(\matr{W}^{t}_{ab})}{\matr{W}_{ab}^{t}}
		\end{aligned}
		\end{equation}
		is an auxiliary function for $\mathscr{F}_{ab}(\matr{W})$.
	\end{lemma}
	
	\begin{proof}
		$\mathscr{G}_{ab}(\matr{W},\matr{W}) = \mathscr{F}_{ab}(\matr{W})$ is obvious. Apply Taylor Expansion to $\mathscr{F}_{ab}(\matr{W})$ on the point $\matr{W}_{ab}^{t}$, we obtain
		\begin{equation}
		\label{F1}
		\begin{aligned}
		\mathscr{F}_{ab}(\matr{W}) & = \mathscr{F}_{ab}(\matr{W}^{t}) + (\matr{W}_{ab} - \matr{W}_{ab}^{t})\mathscr{F}_{ab}^{\prime}(\matr{W}^{t}) \\
		&\quad + \frac{1}{2}(\matr{W}_{ab} - \matr{W}_{ab}^{t})^{2}F_{ab}^{\prime\prime}(\matr{W}^{t}).
		\end{aligned}
		\end{equation}
		Comparing Eq.(\ref{G1}) and Eq.(\ref{F1}), in order to prove $\mathscr{G}_{ab}(\matr{W},\matr{W}^{t}) \geqslant \mathscr{F}_{ab}(\matr{W})$, we only need to prove
		\begin{equation}
		\label{ineq}
		\frac{2[(\matr{W}^{t}\matr{H}\matr{H}^{T})_{ab} +  \lambda_{\matr{W}}C_{row}^{+}(\matr{W}^{t}_{ab})]}{\matr{W}_{ab}^{t}} \geqslant F_{ab}^{\prime\prime}(\matr{W}^{t}).
		\end{equation}
		Rewrite $\mathscr{F}(\matr{W})$ as an addition of two functions (omitting the penalties for $\matr{H}$ since it is irrelevant with $\matr{W}$)
		\begin{equation}
		\begin{aligned}
		\mathscr{F}(\matr{W}) & = \lVert \matr{V} - \matr{W}\matr{H} \rVert_{F}^{2} \\
		&\; + \lambda_{\matr{W}} \sum_{l = 1}^{l_{\matr{W}}} [\exp( E(\matr{W}_{q^{l}:},\matr{W}_{r^{l}:})) + \exp(-E(\matr{W}_{q^{l}:}, \matr{W}_{s^{l}:}))] \\
		& = \mathscr{U}(\matr{W}) + \lambda_{\matr{W}}\mathscr{V}(\matr{W}),
		\end{aligned}
		\end{equation}
		Then Eq.(\ref{ineq}) becomes
		\begin{equation}
		\frac{\mathscr{U}^{\prime}_{ab}(\matr{W}^{t}) + \lambda_{\matr{W}}\mathscr{V}^{\prime+}_{ab}(\matr{W}^{t})}{\matr{W}_{ab}^{t}} \geqslant \mathscr{U}^{\prime\prime}_{ab}(\matr{W}^{t}) + \lambda_{\matr{W}}\mathscr{V}^{\prime\prime}_{ab}(\matr{W}^{t}),
		\end{equation}
		where $\mathscr{V}^{\prime+}_{ab}(\matr{W}^{t})$ is the positive part of $\mathscr{V}^{\prime}_{ab}(\matr{W}^{t})$. Since
		\begin{equation}
		\begin{aligned}
		&\quad \frac{\mathscr{U}^{\prime}_{ab}(\matr{W}^{t})}{\matr{W}_{ab}^{t}} - \mathscr{U}^{\prime\prime}_{ab}(\matr{W}^{t}) \\
		& = \frac{(\matr{W}^{t}\matr{H}\matr{H}^{T})_{ab}}{\matr{W}_{ab}^{t}} - (\matr{H}\matr{H}^{T})_{bb} \\
		& = \frac{\sum_{k=1}^{K} \matr{W}_{ak}^{t} (\matr{H}\matr{H}^{T})_{kb}}{\matr{W}_{ab}^{t}} - (\matr{H}\matr{H}^{T})_{bb} \\
		& \geqslant \frac{\matr{W}_{ab}^{t} (\matr{H}\matr{H}^{T})_{bb}}{\matr{W}_{ab}^{t}} - (\matr{H}\matr{H}^{T})_{bb} \\
		& \geqslant 0,
		\end{aligned}
		\end{equation}
		the only remaining work is to prove
		\begin{equation}
		\frac{\mathscr{V}^{\prime+}_{ab}(\matr{W}^{t})}{\matr{W}_{ab}^{t}} \geqslant \mathscr{V}^{\prime\prime}_{ab}(\matr{W}^{t}).
		\end{equation}
		Recall that the function $\mathscr{V}(\matr{W})$ is a summation of additions of two natural exponential functions, and the exponents are quadratic functions of $\matr{W}$. Calculation of their derivatives is intricate and causes difficulties to prove the above inequality. Here we give a proof by applying approximations to the natural exponential functions.\\
		For the general natural exponential function $f(x) = \mathrm{e}^{x}$ with support $[0,+\infty)$, we can use a piecewise function to approximate it
		\begin{equation}
		\mathrm{e}^{x} \approx 
		\begin{cases}
		(\mathrm{e}^{1} - \mathrm{e}^{0})x + \mathrm{e}^{0},& x \in [0,1) \\
		(\mathrm{e}^{2} - \mathrm{e}^{1})x - \mathrm{e}^{2} + 2\mathrm{e}^{1}, & x \in[1,2)\\
		(\mathrm{e}^{3} - \mathrm{e}^{2})x - 2\mathrm{e}^{3} + 3\mathrm{e}^{2}, & x \in[2,3)\\
		\vdots \\
		(\mathrm{e}^{n} - \mathrm{e}^{n-1})x - (n-1)\mathrm{e}^{n} + n\mathrm{e}^{n-1}, & x \in[n-1,n) \\
		\vdots
		\end{cases},
		\end{equation}
		where each sub-function is a linear function of $x$. The relationship between $f(x)$ and the piecewise function is demonstrated in Fig. \ref{ex}.
		\begin{figure}[htbp]
			\begin{center}
				\includegraphics[width=0.32\textwidth]{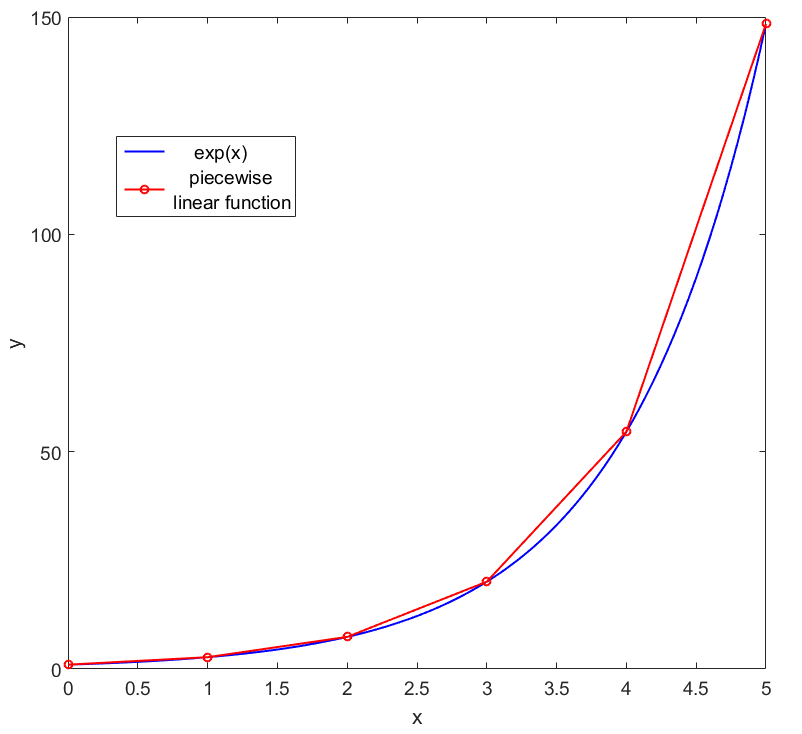}
				\caption{Using a piecewise linear function to approximate $\mathrm{e}^{x}$.\label{ex}}
			\end{center}
		\end{figure}
		Similarly, we can approximate $g(x) = \mathrm{e}^{-x}$ with a piecewise function whose $n^{\text{th}}$ sub-function is
		\begin{equation}	
		(\mathrm{e}^{-n} - \mathrm{e}^{-(n-1)})x - (n-1)\mathrm{e}^{-n} + n\mathrm{e}^{-(n-1)},  x \in[n-1,n).
		\end{equation}
		If the support $[0,+\infty)$ is divided into infinite intervals, the piecewise functions will be infinitely close to $f(x)$ and $g(x)$.\\
		Now we can approximate $\mathscr{V}(\matr{W})$ using the above two piecewise functions. The $n^{\text{th}}$ part of it is 
		\begin{equation}
		\begin{aligned}
		& \mathscr{V}(\matr{W})^{(n)} \approx \sum_{l = 1}^{l_{\matr{W}}} \bigg\{ \\
		& (\mathrm{e}^{n} - \mathrm{e}^{n-1})E(\matr{W}_{q^{l}:},\matr{W}_{r^{l}:}) + (\frac{1}{\mathrm{e}^{n}} - \frac{1}{\mathrm{e}^{n-1}})E(\matr{W}_{q^{l}:}, \matr{W}_{s^{l}:})\\ 
		& - (n-1)\mathrm{e}^{n} + n\mathrm{e}^{n-1} - (n-1)\mathrm{e}^{-n} + n\mathrm{e}^{-(n-1)}\bigg\}.
		\end{aligned}
		\end{equation}
		Its first and second order derivatives w.r.t $\matr{W}_{ab}$ are
		\begin{equation}
		\begin{aligned}
		& \mathscr{V}^{\prime}_{ab}(\matr{W})^{(n)} \approx \sum_{l = 1}^{l_{\matr{W}}} \bigg\{\\
		& (\mathrm{e}^{n} - \mathrm{e}^{n-1})[\sum_{q^{l} = a} (\matr{W}_{q^{l}b} - \matr{W}_{r^{l}b}) + \sum_{r^{l} = a} (\matr{W}_{r^{l}b} - \matr{W}_{q^{l}b})] +\\
		& (\frac{1}{\mathrm{e}^{n}} - \frac{1}{\mathrm{e}^{n-1}})[\sum_{q^{l} = a} (\matr{W}_{q^{l}b} - \matr{W}_{s^{l}b}) + \sum_{s^{l} = a} (\matr{W}_{s^{l}b} - \matr{W}_{q^{l}b})] \bigg\},
		\end{aligned}
		\end{equation}
		\begin{equation}
		\begin{aligned}
		& \mathscr{V}^{\prime\prime}_{ab}(\matr{W})^{(n)} \approx \sum_{l = 1}^{l_{\matr{W}}} \bigg\{(\mathrm{e}^{n} - \mathrm{e}^{n-1})(\sum_{q^{l} = a} 1 + \sum_{r^{l} = a} 1) +\\
		& (\frac{1}{\mathrm{e}^{n}} - \frac{1}{\mathrm{e}^{n-1}})(\sum_{q^{l} = a} 1 + \sum_{s^{l} = a} 1) \bigg\}.
		\end{aligned}
		\end{equation}
		And the positive part of the first derivative is
		\begin{equation}
		\begin{aligned}
		& \mathscr{V}^{\prime+}_{ab}(\matr{W})^{(n)} \approx \sum_{l = 1}^{l_{\matr{W}}} \bigg\{(\mathrm{e}^{n} - \mathrm{e}^{n-1})(\sum_{q^{l} = a} \matr{W}_{q^{l}b} + \sum_{r^{l} = a} \matr{W}_{r^{l}b}) -\\
		& (\frac{1}{\mathrm{e}^{n}} - \frac{1}{\mathrm{e}^{n-1}})(\sum_{q^{l} = a}  \matr{W}_{s^{l}b} + \sum_{s^{l} = a} \matr{W}_{q^{l}b}) \bigg\}.
		\end{aligned}
		\end{equation}
		Then we have
		\begin{equation}
		\begin{aligned}
		&\quad \frac{\mathscr{V}^{\prime+}_{ab}(\matr{W}^{t})^{(n)}}{\matr{W}_{ab}^{t}} - \mathscr{V}^{\prime\prime}_{ab}(\matr{W}^{t})^{(n)} \\
		& = \sum_{l = 1}^{l_{\matr{W}}} \bigg\{ 
		(\frac{1}{\mathrm{e}^{n-1}} - \frac{1}{\mathrm{e}^{n}})(\sum_{q^{l} = a}  (\frac{\matr{W}_{s^{l}b}^{t}}{\matr{W}_{q^{l}b}^{t}} + 1) + \sum_{s^{l} = a} (\frac{\matr{W}_{q^{l}b}^{t}}{\matr{W}_{s^{l}b}^{t}} + 1)) \bigg\} \\
		& \geqslant 0
		\end{aligned}
		\end{equation}
		Thus Eq.(\ref{ineq}) holds.
	\end{proof}
	
	Now we can prove the convergence of Theorem \ref{theorem1}:
	\begin{proof}[Proof of Theorem \ref{theorem1}]
		According to Eq.(\ref{arg}) and Eq.(\ref{G1}), we get
		\begin{equation}
		\begin{aligned}
		\matr{W}_{ab}^{t+1} & = \argmin_{\matr{W}_{ab}} \mathscr{G}_{ab}(\matr{W}, \matr{W}^{t}) \\
		& = \matr{W}_{ab}^{t}\frac{(\matr{V}\matr{H}^{T})_{ab} + \lambda_{\matr{W}} C_{row}^{-}(\matr{W}_{ab})}{(\matr{W}\matr{H}\matr{H}^{T})_{ab} + \lambda_{\matr{W}} C_{row}^{+}(\matr{W}_{ab})}.
		\end{aligned}
		\end{equation}
		As $\mathscr{G}_{ab}(\matr{W},\matr{W}^{t})$ is an auxiliary function, $\mathscr{F}_{ab}(\matr{W})$ is non-increasing under this update rule.
	\end{proof}
	
	\subsubsection{Convergence of Divergence updating rules and Theorem 2}
	\label{proof2}
	As for updating rules in Eq.(\ref{updateW2}) \& (\ref{updateH2}), we have
	\begin{lemma}
		Function
		\begin{equation}
		\label{G2}
		\begin{aligned}
		\mathscr{G}_{ab}&(\matr{W}, \matr{W}^{t}) = \sum_{j}(\matr{V}_{aj} \log \matr{V}_{aj} - \matr{V}_{aj}) + \sum_{j}(\matr{W}\matr{H})_{aj} \\
		- & \sum_{jk} \matr{V}_{aj}\frac{\matr{W}_{ak}^{t}\matr{H}_{kj}}{(\matr{W}^{t}\matr{H})_{aj}}(\log \matr{W}_{ak}\matr{H}_{kj} - \log \frac{\matr{W}_{ak}^{t}\matr{H}_{kj}}{(\matr{W}^{t}\matr{H})_{aj}}) \\
		+ & \lambda_{\matr{W}} \sum_{l = 1}^{l_{\matr{W}}} \max(0, SD(\matr{W}_{q^{l}:},\matr{W}_{r^{l}:}) - SD(\matr{W}_{q^{l}:}, \matr{W}_{s^{l}:}) \\
		+ & \lambda_{\matr{H}} \sum_{l = 1}^{l_{\matr{H}}} \max(0, SD(\matr{H}_{:q^{l}},\matr{H}_{:r^{l}}) - SD(\matr{H}_{:q^{l}}, \matr{H}_{:s^{l}}). 
		\end{aligned}
		\end{equation}
		This is an auxiliary function for $\mathscr{F}_{ab}(\matr{W})$.
	\end{lemma}
	
	\begin{proof}
		$\mathscr{G}_{ab}(\matr{W},\matr{W}) = \mathscr{F}_{ab}(\matr{W})$ is obvious. To prove that $\mathscr{G}_{ab}(\matr{W},\matr{W}^{t}) \geqslant \mathscr{F}_{ab}(\matr{W})$, we utilize the Jensen Inequality to obtain
		\begin{equation}
		-\log \sum_{k} \matr{W}_{ak}\matr{H}_{kj} \leqslant - \sum_{k} \alpha_{k} \log \frac{\matr{W}_{ak}\matr{H}_{kj}}{\alpha_{k}}
		\end{equation}
		which holds for all non-negative $\alpha_{k}$ that sum to unity. Setting
		\begin{equation}
		\alpha_{k} = \frac{\matr{W}_{ak}^{t}\matr{H}_{kj}}{(\matr{W}^{t}\matr{H})_{aj}}
		\end{equation}
		we obtain
		\begin{equation}
		\begin{aligned}
		-\sum_{j}\matr{V}_{aj}\log \sum_{k} \matr{W}_{ak}\matr{H}_{kj} & \leqslant \\
		-\sum_{jk} \matr{V}_{aj}\frac{\matr{W}_{ak}^{t}\matr{H}_{kj}}{(\matr{W}^{t}\matr{H})_{aj}}(\log \matr{W}_{ak}\matr{H}_{kj} & - \log \frac{\matr{W}_{ak}^{t}\matr{H}_{kj}}{(\matr{W}^{t}\matr{H})_{aj}}).
		\end{aligned}
		\end{equation}
		This is the only different part between $\mathscr{F}_{ab}(\matr{W})$ and $\mathscr{G}_{ab}(\matr{W},\matr{W}^{t})$. Thus, $\mathscr{F}_{ab}(\matr{W}) \leqslant \mathscr{G}_{ab}(\matr{W},\matr{W}^{t})$ holds.
	\end{proof}
	
	Now we can prove the convergence of Theorem \ref{theorem2}:
	\begin{proof}[Proof of Theorem \ref{theorem2}]
		According to Eq.(\ref{arg}) and Eq.(\ref{G2}), we get
		\begin{equation}
		\begin{aligned}
		\matr{W}_{ab}^{t+1} & = \argmin_{\matr{W}_{ab}} \mathscr{G}_{ab}(\matr{W}, \matr{W}^{t}) \\
		& \approx \matr{W}_{ab}^{t}\frac{\sum_{j}V_{aj}\matr{H}_{bj}/(\matr{WH})_{aj}}{\frac{1}{2}\lambda_{\matr{W}}P_{row}(\matr{W}_{ab}) + \sum_{j}\matr{H}_{bj}}.
		\end{aligned}
		\end{equation}
		Notice that $\matr{W}$ also exists in the penalty regularisation terms, and it is difficult to calculate their corresponding derivatives with respect to $\matr{W}$. For simplicity and efficiency, we substitute $\matr{W}$ in penalties with $\matr{W}^{t}$ so that they become irrelevant to $\matr{W}$. This approximation of the local optima of $\mathscr{G}_{ab}(\matr{W},\matr{W}^{t})$ is numerically proved feasible through our experiments.\\
		As $\mathscr{G}_{ab}(\matr{W},\matr{W}^{t})$ is an auxiliary function, $\mathscr{F}_{ab}(\matr{W})$ is non-increasing under this update rule.
	\end{proof}
	
	The complete algorithms of RPR-NMF are demonstrated in Algorithm \ref{PRNMF_euc} and Algorithm \ref{PRNMF_div}.
	
	\begin{algorithm}
		\caption{Relative Pairwise Relationship constrained Non-negative Matrix Factorisation (RPR-NMF) using Euclidean measure}
		\label{PRNMF_euc}
		\begin{algorithmic}[1]
			\REQUIRE $\matr{V} \in \mathbb{R}^{N \times M}_{\geqslant 0}$ (factorising matrix),\\
			$K \in \mathbb{N}^{+}$ (latent dimension),\\
			$\matr{L}_{\matr{W}} \in \mathbb{N}^{l_{\matr{W}} \times 3}$ (pairwise relationship constraints on $\matr{W}$),\\
			$\matr{L}_{\matr{H}} \in \mathbb{N}^{l_{\matr{H}} \times 3}$ (pairwise relationship constraints on $\matr{H}$),\\
			$\lambda_{\matr{W}} \in \mathbb{R}^{+}, \lambda_{\matr{H}} \in \mathbb{R}^{+}$ (penalty coefficients)\\
			\ENSURE $\matr{W} \in \mathbb{R}^{N \times K}_{\geqslant 0}$ (the left factorised matrix),\\
			$\matr{H} \in \mathbb{R}^{K \times M}_{\geqslant 0}$ (the right factorised matrix)\\
			\STATE Initialise $\matr{W}$ and $\matr{H}$ as non-negative random matrices
			\WHILE{Terminal conditions not satisfied}
			\FOR{$k = 1$ \TO $K$}
			\STATE // Update $\matr{W}$
			\FOR{$a = 1$ \TO $N$}
			\STATE Calculate $C_{row}^{+}(\matr{W}_{ab})$ and $C_{row}^{-}(\matr{W}_{ab})$ as Eq.(\ref{CW})	
			\STATE $\matr{W}_{ak} \leftarrow \matr{W}_{ak}\frac{(\matr{V}\matr{H})_{ak} + \lambda_{\matr{W}}C_{row}^{-}(\matr{W}_{ab})}{(\matr{W}\matr{H}\matr{H}^{T})_{ak} + \lambda_{\matr{W}}C_{row}^{+}(\matr{W}_{ab})}$
			\ENDFOR
			\STATE // Update $\matr{H}$
			\FOR{$b = 1$ \TO $M$}
			\STATE Calculate $C_{col}^{+}(\matr{H}_{ab})$ and $C_{col}^{-}(\matr{H}_{ab})$ as Eq. (\ref{CW})	
			\STATE $\matr{H}_{kb} \leftarrow \matr{H}_{kb}\frac{(\matr{W}^{T}\matr{V})_{kb} + \lambda_{\matr{H}} C_{col}^{-}(\matr{H}_{ab})}{(\matr{W}^{T}\matr{W}\matr{H})_{kb} + \lambda_{\matr{H}} C_{col}^{-}(\matr{H}_{ab})}$
			\ENDFOR
			\ENDFOR
			\ENDWHILE
		\end{algorithmic}
	\end{algorithm}
	
	\begin{algorithm}
		\caption{Relative Pairwise Relationship constrained Non-negative Matrix Factorisation (RPR-NMF) using Divergence measure}
		\label{PRNMF_div}
		\begin{algorithmic}[1]
			\REQUIRE $\matr{V} \in \mathbb{R}^{N \times M}_{\geqslant 0}$ (factorising matrix),\\
			$K \in \mathbb{N}$ (latent dimension),\\
			$\matr{L}_{\matr{W}} \in \mathbb{N}^{l_{\matr{W}} \times 3}$ (pairwise relationship constraints on $\matr{W}$),\\
			$\matr{L}_{\matr{H}} \in \mathbb{N}^{l_{\matr{H}} \times 3}$ (pairwise relationship constraints on $\matr{H}$),\\
			$\lambda_{\matr{W}} \in \mathbb{R}^{+}, \lambda_{\matr{H}} \in \mathbb{R}^{+}$ (penalty coefficients)\\
			\ENSURE $\matr{W} \in \mathbb{R}^{N \times K}_{\geqslant 0}$ (the left factorised matrix),\\
			$\matr{H} \in \mathbb{R}^{K \times M}_{\geqslant 0}$ (the right factorised matrix)\\
			\STATE Initialise $\matr{W}$ and $\matr{H}$ as non-negative random matrices
			\WHILE{Terminal conditions not satisfied}
			\FOR{$k = 1$ \TO $K$}
			\STATE // Update $\matr{W}$
			\FOR{$a = 1$ \TO $N$}
			\STATE Calculate $P_{row}(\matr{W}_{ab})$ as Eq.(\ref{PW})
			\IF{$\frac{1}{2}\lambda_{\matr{W}}P_{row}(\matr{W}_{ab}) + \sum_{j}\matr{H}_{kj} < 0$}
			\STATE $\matr{W}_{ak} \leftarrow \matr{W}_{ak}\frac{\sum_{j}\matr{V}_{aj}\matr{H}_{kj}/(\matr{W}\matr{H})_{aj}}{\sum_{j}\matr{H}_{kj}}$
			\ELSE
			\STATE $\matr{W}_{ak} \leftarrow \matr{W}_{ak}\frac{\sum_{j}\matr{V}_{aj}\matr{H}_{kj}/(\matr{W}\matr{H})_{aj}}{\frac{1}{2}\lambda_{\matr{W}}P_{row}(\matr{W}_{ab}) + \sum_{j}\matr{H}_{kj}}$
			\ENDIF
			\ENDFOR
			\STATE // Update $\matr{H}$
			\FOR{$b = 1$ \TO $M$}
			\STATE Calculate $P_{col}(\matr{H}_{ab})$ as Eq.(\ref{PW})
			\IF{$\frac{1}{2}\lambda_{\matr{H}}P_{col}(\matr{H}_{ab}) + \sum_{i}\matr{W}_{ik} < 0$}
			\STATE $\matr{H}_{kb} \leftarrow \matr{H}_{kb}\frac{\sum_{i}\matr{V}_{ib}\matr{W}_{ik}/(\matr{W}\matr{H})_{ib}}{\sum_{i}\matr{W}_{ik}}$
			\ELSE	
			\STATE $\matr{H}_{kb} \leftarrow \matr{H}_{kb}\frac{\sum_{i}\matr{V}_{ib}\matr{W}_{ik}/(\matr{W}\matr{H})_{ib}}{\frac{1}{2}\lambda_{\matr{H}}P_{col}(\matr{H}_{ab}) + \sum_{i}\matr{W}_{ik}}$
			\ENDIF
			\ENDFOR
			\ENDFOR
			\IF{Objective function value increases}
			\STATE $\lambda_{\matr{W}} \leftarrow 1/2 \cdot \lambda_{\matr{W}}, \quad \lambda_{\matr{H}} \leftarrow 1/2 \cdot \lambda_{\matr{H}}$
			\STATE Roll back
			\ELSE
			\STATE $\lambda_{\matr{W}} \leftarrow 1.01 \cdot \lambda_{\matr{W}}, \quad \lambda_{\matr{H}} \leftarrow 1.01 \cdot \lambda_{\matr{H}}$
			\ENDIF 
			\ENDWHILE
		\end{algorithmic}
	\end{algorithm}
	
	\section{Experiments}
	
	In this section, we conducted experiments on both synthetic datasets and real datasets to demonstrate the performance of our proposed algorithm RPR-NMF. The baseline algorithms we chose for comparison are: NMF \cite{LeeNIPS2001}, GNMF \cite{CaiPAMI2011}, and LCNMF \cite{LiuPAMI2012}. Each of the algorithms have two versions: one with Euclidean measure (denoted by suffix ``\_euc'' in figures and tables), the other with Divergence measure (denoted by suffix ``\_div''). As we stated in the previous section that GNMF and LCNMF are both designed as dimensionality reduction methods, the weight matrix for GNMF only performs as constraints imposed on the right factorised matrix and the label matrix for LCNMF also only affects the right factorised matrix. Thus in our experiments on synthetic datasets as well as datasets for image clustering, RPR-NMF only impose RPR constraints on the right factorised matrix. However, when it comes down to recommender systems, we impose constraints on both factorised matrices which are considered as users' and items' features.
	
	The metrics we used to evaluate the performance of algorithms are:
	
	1. \textit{Mean Squared Loss (MSL)} for algorithms using Euclidean measure and \textit{Mean Divergence (MD)} for algorithms using Divergence measure, which evaluate the approximation performance and are defined as
	\begin{eqnarray}
	\text{MSL} & = & \frac{1}{NM}\lVert \matr{V} - \matr{W}^{\ast}\matr{H}^{\ast} \rVert_{F}^{2},\\
	\text{MD} & = & \frac{1}{NM}D(\matr{V}||\matr{W}^{\ast}\matr{H}^{\ast}),
	\end{eqnarray}
	where $\matr{W}^{\ast}$ and $\matr{H}^{\ast}$ are the final factorised matrices.
	
	2. \textit{Constraint Satisfied Rate (CSR)}, which presents how well the RPR constraints are satisfied and is defined as
	\begin{equation}
	\text{CSR} = \frac{1}{2}(l_{\matr{W}}^{\ast}/l_{\matr{W}}+l_{\matr{H}}^{\ast}/l_{\matr{H}}),
	\end{equation}
	where $l_{\matr{W}}^{\ast}$ and $l_{\matr{H}}^{\ast}$ are the numbers of satisfied constraints.
	
	3. \textit{Clustering Accuracy (ACC)}, which we calculate by constructing a cost matrix, solving the matrix by Munkres Assign Algorithm \cite{MunkresSIAM1957}, and mapping one clustering to the other; and \textit{Normalised Mutual Information (NMI)} \cite{XuSIGIR2003}, which is to calculate the entropy information correlation of clusterings. They are used to evaluate the performance of the clustering results in image clustering experiments.
	
	4. \textit{Root Mean Squared Error (RMSE)} \cite{BarnstonWF1992}, which evaluates the difference between recovering ratings and original ratings (the latter are removed before factorisation in cross validation); and \textit{F1 score} \cite{PowersBP2011}, which tells how much proportion of correct recommendations the recovered rating matrix gives. They are used to evaluate the performance on recommendation systems and defined as
	\begin{equation}
	\text{RMSE} = (\sum_{ij}\matr{M}_{ij})^{-1/2}\lVert \matr{M} \ast (\matr{V} - \matr{W}^{\ast}\matr{H}^{\ast}) \rVert_{F},
	\end{equation}
	\begin{equation}
	\text{F1} = \frac{2 \times \text{Precision} \times \text{Recall}}{\text{Precision} + \text{Recall}},
	\end{equation}
	where $\matr{M}$ is a binary matrix in which ones denote the chosen entries and zeros denote not chosen entries in a cross validation experiment. As for each user, we use its average rating as a threshold. Ratings greater than the threshold suggest the corresponding items are recommended. Thus we obtain two recommendation binary vectors before and after factorisation, which can be used to calculate the value of Precision and Recall.
	
	We first introduce two algorithms used to convert the additional information. Then we validate the effectiveness of RPR-NMF and analyse the computing complexity through two synthetic experiments, followed by experiments on real datasets for applications of image clustering and recommender systems. The statistics of the datasets we used are presented in TABLE \ref{st}.
	
	All the code of RPR-NMF as well as preprocessed datasets can be downloaded from: https://github.com/shawn-jiang/RPRNMF.
	
	\begin{table}[!htbp]
		\caption{Statistics of Datasets used for experiments on Effectiveness Validation, Image Clustering and Recommender Systems}
		\begin{center}
			\label{st}
			\begin{tabular}{|c|c|c|c|c|}
				\hline
				dataset & rows & columns & non-zeros & density\\
				\hline
				Synthetic 1 & 100 & 100 & 10,000 & 1\\
				Synthetic 2 & $[20,200]$ & $[20,200]$ & $[400,40000]$ & 1\\
				\hline
				AT\&T ORL & 1,024 & 400 & 409,600 & 1\\
				CMU PIE & 1,024 & 2,856 & 2,924,544 & 1\\
				Movielens 1M & 6,040 & 3,706 & 1,000,209 & 0.0447\\
				\hline
			\end{tabular}
		\end{center}
	\end{table}
	
	\subsection{Additional Information Conversion}
	Since GNMF, LCNMF and RPR-NMF all need additional information besides the factorising matrix itself, it is necessary to make it fair for all the three algorithms to equally access available additional information.
	
	\begin{algorithm}
		\caption{Constructing Weight Matrix for GNMF from Pairwise Relationship Constraints}
		\label{smx}
		\begin{algorithmic}[1]
			\REQUIRE $M \in \mathbb{N}^{+}$ (the column dimension of $\matr{H}$),\\
			$\matr{L}_{\matr{H}} \in \mathbb{N}^{l_{\matr{H}} \times 3}$ (pairwise relationship constraints on $\matr{H}$),\\
			$mins, maxs \in \mathbb{R}$ (minimal and maximal weight to set)
			\ENSURE $\matr{S}_{\matr{H}} \in \mathbb{R}^{M \times M}_{\geqslant 0}$ (weight matrix for GNMF)
			\STATE $\matr{S}_{\matr{H}} \leftarrow \matr{I}_{M}$
			\STATE Construct a network initialised with zero node
			\FOR{$l = 1$ \TO $l_{\matr{H}}$}
			\STATE $node_{1} \leftarrow (\matr{L}_{\matr{H}}(l,1),\matr{L}_{\matr{H}}(l,2))$
			\STATE $node_{2} \leftarrow (\matr{L}_{\matr{H}}(l,1),\matr{L}_{\matr{H}}(l,3))$
			\STATE Put the two nodes into the network (if they have not been there) with an directed edge from $node_{1}$ to $node_{2}$
			\ENDFOR
			\STATE Recursively calculate the depth of each node (nodes with no outgoing edges are set depth $1$, their direct parents are set depth $2$, and so on).
			\STATE $t \leftarrow (maxs - mins) / (\max(depth) - 1)$
			\FOR{each node $n$}
			\STATE $i,j \leftarrow$ two points of $n$
			\STATE $\matr{S}_{\matr{H}}(i,j) \leftarrow mins + (depth(n) - 1) * t$
			\ENDFOR
		\end{algorithmic}
	\end{algorithm}
	
	\begin{algorithm}
		\caption{Constructing Label Matrix for LCNMF from Pairwise Relationship Constraints}
		\label{lmx}
		\begin{algorithmic}[1]
			\REQUIRE $M \in \mathbb{N}^{+}$ (the column dimension of $\matr{H}$),\\
			$\matr{L}_{\matr{H}} \in \mathbb{N}^{l_{\matr{H}} \times 3}$ (pairwise relationship constraints on $\matr{H}$)
			\ENSURE $\matr{B}$ (binary matrix whose size is at most $M \times M$ but not determined)
			\STATE $\matr{B} \leftarrow \matr{0}_{M}$
			\FOR{$l = 1$ \TO $l_{\matr{H}}$}
			\STATE $q \leftarrow \matr{L}_{\matr{H}}(l,1)$ \quad $r \leftarrow \matr{L}_{\matr{H}}(l,2)$
			\IF{$\matr{B}_{:,q} \neq \bm{0} \land \matr{B}_{:,r} \neq \bm{0}$}
			\STATE $\matr{B}_{:,q} \leftarrow \matr{B}_{:,q} \lor \matr{B}_{:,r}$
			\STATE $\matr{B}_{:,r} \leftarrow \bm{0}$
			\ELSIF{$\matr{B}_{:,q} \neq \bm{0}$}
			\STATE $label(r) \leftarrow label(q)$
			\ELSIF{$\matr{B}_{:,r} \neq \bm{0}$}
			\STATE $label(q) \leftarrow label(r)$
			\ELSE
			\STATE $q,r \leftarrow$ new label
			\ENDIF
			\ENDFOR
			\STATE Delete rows containing all zeros from $\matr{B}$
		\end{algorithmic}
	\end{algorithm}
	
	GNMF needs a weight matrix which describes how close the data vectors should be, LCNMF requires which data vectors should have the same label, while RPR-NMF demands a list of RPR constraints among factorised vectors. These three additional information can be converted from each to the others. However, their intensities are different: labels of LCNMF are the strongest, followed by the similarities of GNMF, and the RPR constraints of RPR-NMF are the weakest. It will cause loss of information if we convert stronger constraints to weaker ones, but not if it is done the other way around. Thus here we introduce two algorithms which convert the RPR constraints for RPR-NMF to the weight matrix for GNMF (Algorithm \ref{smx}) and the label matrix for LCNMF (Algorithm \ref{lmx}) respectively.
	
	\subsection{Effectiveness Validation \& Complexity Analysis}
	We conducted two experiments on synthetic datasets to verify the effectiveness and study the computing complexity of RPR-NMF comparing to the baseline algorithms.
	
	It is worth noting that, the RPR constraints in the experiments are randomly generated but there are chains among them. A chain of constraints is like $dis(\bm{a},\bm{b}) < dis(\bm{b},\bm{c}) < dis(\bm{c},\bm{d})$, which is a $3$-chain of constraints. As we are going to show in the results, RPR-NMF can satisfy chain constraints while others are incapable.
	
	\begin{figure*}[htbp]
		\begin{subfigure}[t]{0.32\linewidth}
			\begin{center}
				\includegraphics[width=\textwidth]{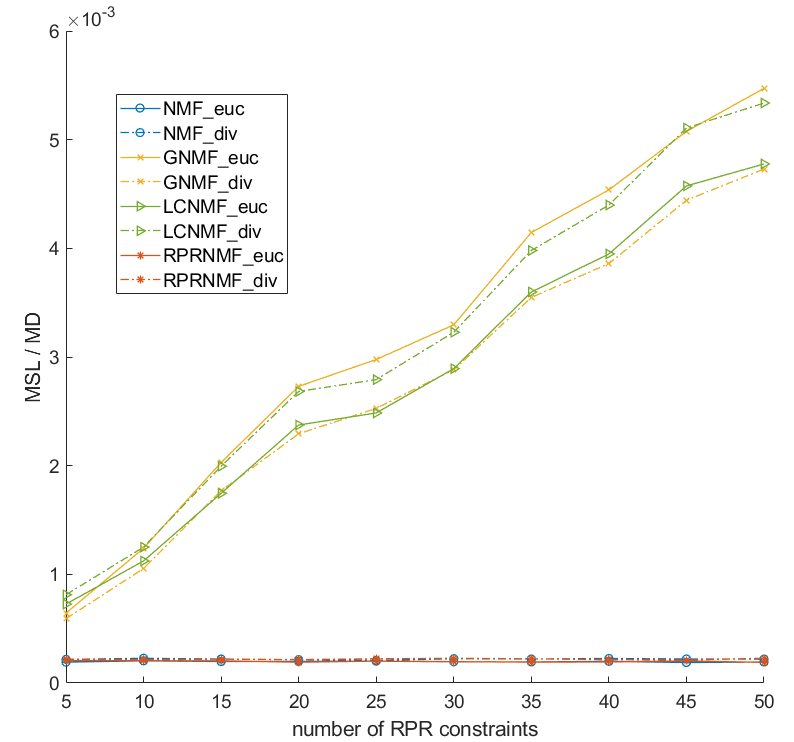}
				\caption{}
			\end{center}
		\end{subfigure}
		\hspace{1ex}
		\begin{subfigure}[t]{0.32\linewidth}
			\begin{center}
				\includegraphics[width=\textwidth]{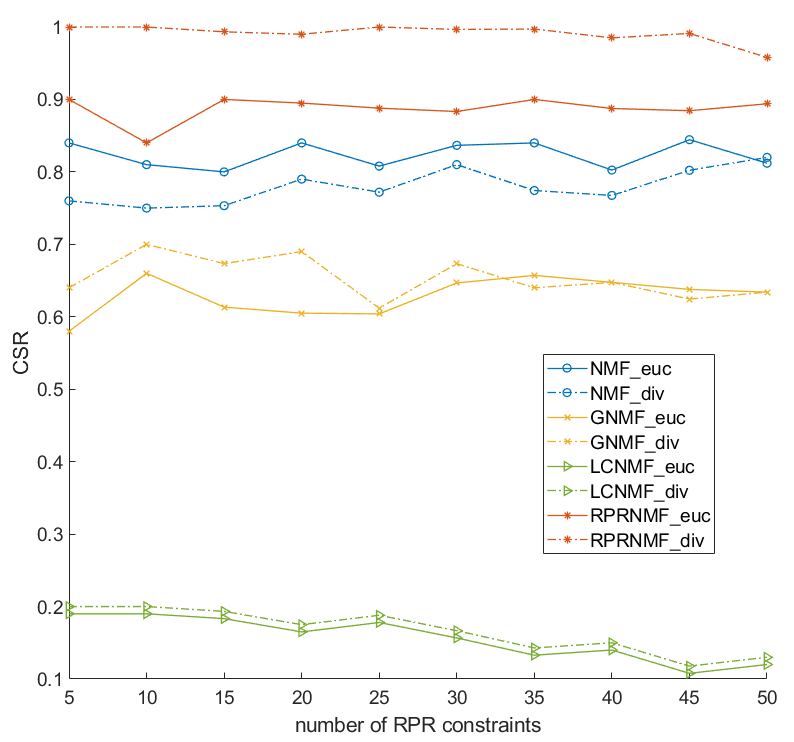}
				\caption{}
			\end{center}
		\end{subfigure}
		\hspace{1ex}
		\begin{subfigure}[t]{0.32\linewidth}
			\begin{center}
				\includegraphics[width=\textwidth]{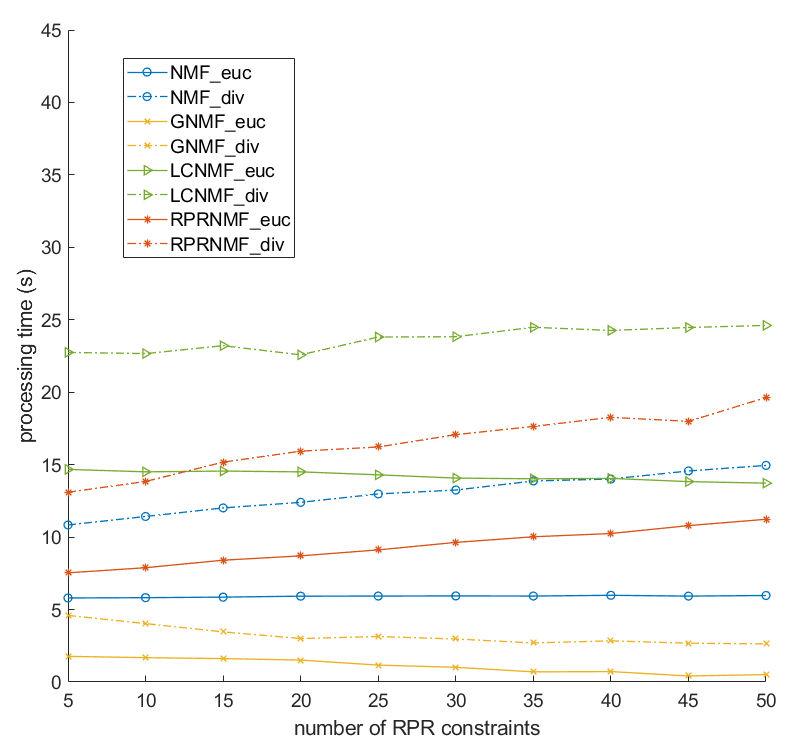}
				\caption{}
			\end{center}
		\end{subfigure}
		\caption{Performance with respect to the number of pairwise relationship constraints. (a) MSL/MD, (b) CSR, (c) Processing time. \label{syn1}}
	\end{figure*}
	
	\begin{figure*}[htbp]
		\begin{subfigure}[t]{0.32\linewidth}
			\begin{center}
				\includegraphics[width=\textwidth]{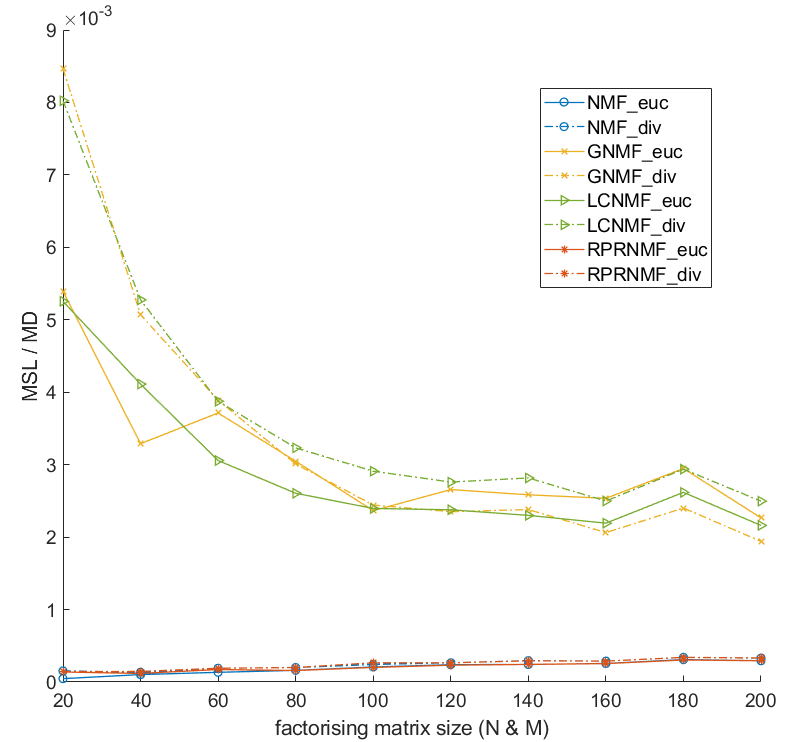}
				\caption{}
			\end{center}
		\end{subfigure}
		\hspace{1ex}
		\begin{subfigure}[t]{0.32\linewidth}
			\begin{center}
				\includegraphics[width=\textwidth]{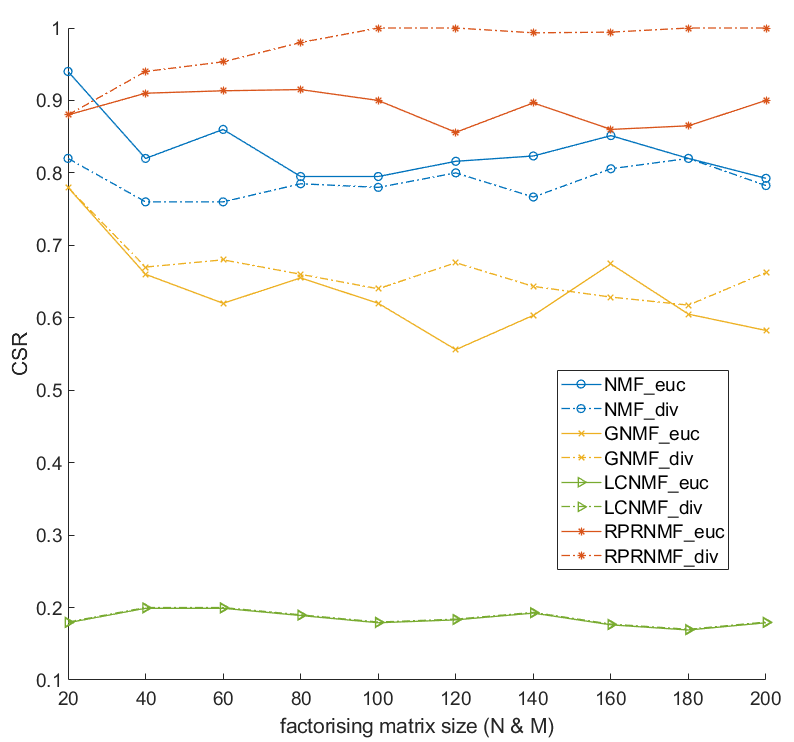}
				\caption{}
			\end{center}
		\end{subfigure}
		\hspace{1ex}
		\begin{subfigure}[t]{0.32\linewidth}
			\begin{center}
				\includegraphics[width=\textwidth]{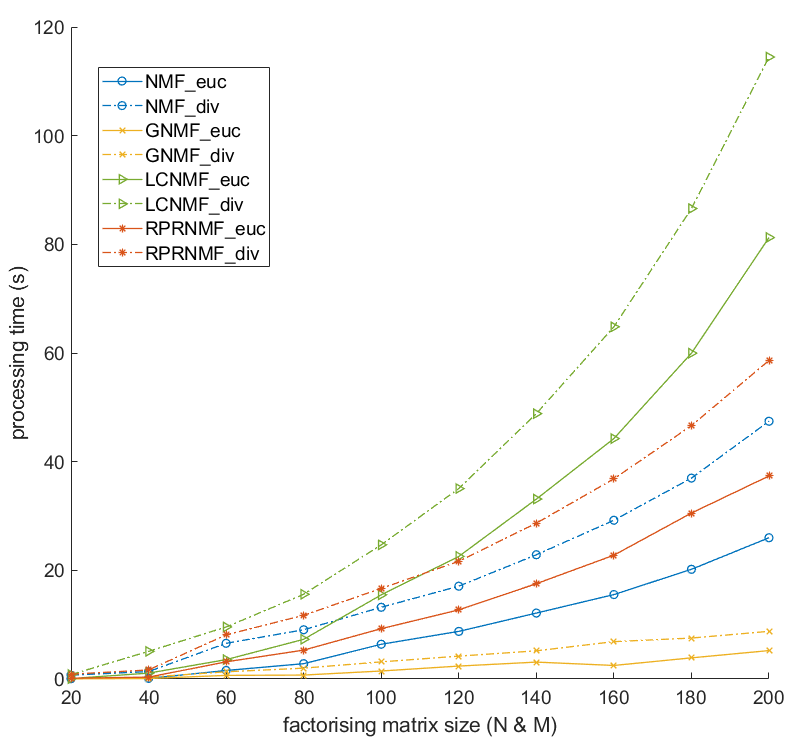}
				\caption{}
			\end{center}
		\end{subfigure}
		\caption{Performance with respect to the size of factorising matrix. (a) MSL/MD, (b) CSR, (c) Processing time. \label{syn2}}
	\end{figure*}
	
	In the first experiment, we explored the influence of the number of RPR constraints with respect to MSL/MD, CSR and processing time. First, we randomly generated a $100 \times 20$ matrix $\matr{W}_{0}$ and a $20 \times 100$ matrix $\matr{H}_{0}$, and got the factorising matrix $\matr{V}$ by their product. Then we generate $1$ to $10$ groups of $5$-chain constraints (i.e. $5$ to $50$ constraints) from $\matr{H}_{0}$. $K$ is set $20$ and $\lambda_{\matr{H}}$ is set $1$. For each experiment with different number of constraints, we repeated it $10$ times and  present the average results in Fig. \ref{syn1}.
	
	As the number of constraints increases, (a) NMF and RPR-NMF algorithms obtain similar and the lowest errors (average $1.963 \times 10^{-4}$ MSL and $2.197 \times 10^{-4}$ MD for NMF, $1.984 \times 10^{-4}$ MSL and $2.195 \times 10^{-4}$ MD for RPR-NMF), while the errors of GNMF and LCNMF are increasing and much higher than those of NMF and RPR-NMF; (b) RPR-NMF using Divergence measure achieves the highest CSRs (average $99.11\%$) followed by its Euclidean version ($88.72\%$), and NMF, GNMF, LCNMF obtain $80.16\%$, $64.10\%$, $16.14\%$ average CSR respectively; (c) the processing time of RPR-NMF and the Euclidean version of NMF is increasing while others' is nearly unchanged or slightly decreasing, and GNMF are the fastest while LCNMF are the slowest when comparing algorithms solely with Euclidean measure or Divergence measure.
	
	The second experiment focused on how the size of factorising matrix affects the performance of different algorithms. The size of the factorising matrix $\matr{V}$ varies from $20$ to $200$ with step $20$. Each $\matr{V}$ is also obtained by the product of randomly generated $\matr{W}_{0}$ and $\matr{H}_{0}$. For each matrix size, the latent dimension $K$ is set $N / 5$, and the number of constraints equals to $K$. The constraints are also of 5-chains. $\lambda_{\matr{H}}$ is set $1$.The average results of $10$ repeated experiments are presented in Fig. \ref{syn2}.
	
	As the size of factorising matrix increases, (a) NMF and RPR-NMF algorithms obtain similar and the lowest errors again with slight rises, while the errors of GNMF and LCNMF are decreasing but still much higher than those of NMF and RPR-NMF; (b) the performance on CSR is similar to that in the first experiment, where RPR-NMF using Divergence measure achieves the highest score; (c) the processing time of all algorithms is increasing, and GNMF are still the fastest while LCNMF are still the slowest respectively for Euclidean measure and Divergence measure.
	
	The results on CSR also suggest that the weight matrix of GNMF and the labels of LCNMF are not helpful on satisfying such RPR constraints. When there are only independent constraints, GNMF and LCNMF can satisfy them by setting proper weights and labels (e.g. for a constraint $dis(\bm{a},\bm{b}) < dis(\bm{b},\bm{c})$, GNMF sets weight $1$ for ($\bm{a}, \bm{b}$) and $0$ for ($\bm{b}, \bm{c}$), and LCNMF sets ($\bm{a}, \bm{b}$) the same label). However, when the constraints are not independent (such as in chains), GNMF cannot simply set the weights $1$s and $0$s, and LCNMF can only guarantee satisfying one constraint of a group of chain constraints (e.g. for a 3-chain of constraints $dis(\bm{a},\bm{b}) < dis(\bm{b},\bm{c}) < dis(\bm{c},\bm{d})$, the weight of ($\bm{b},\bm{c}$) in GNMF has to be greater than $0$ and less than $1$, while in LCNMF, setting ($\bm{a},\bm{b}$) the same label can satisfy the first constraint, but the only way to satisfy the second constraint is to set ($\bm{b},\bm{c}$) the same label, which will contradict with the first constraint). Thus, GNMF and LCNMF cannot guarantee most of the constraints are satisfied after factorisation (recall the example in Fig. \ref{eg}).
	
	The above two synthetic experiments both demonstrate that RPR-NMF algorithms have the advantages of getting accurate factorisation, satisfying relative constraints and being applied to large scale datasets comparing with other algorithms.
	
	\subsection{Parameter Selection}
	The parameters introduced by RPR-NMF are the penalty coefficients. We conducted experiments with varying values of these parameters under the following settings: $N = 100, M = 100, K = 20, l_{\matr{W}} = l_{\matr{H}} = N / 10$. $\lambda_{\matr{W}}$ and $\lambda_{\matr{H}}$ vary from $0.4$ to $4$ with step $0.4$ and from $20$ to $100$ with step $20$. The average results of $10$ repetitions are presented in Fig. \ref{la}.
	
	As showed in the figure, the coefficients have nearly no influence on the CSR for both algorithms. However, the MSL/MD of RPR-NMF using Euclidean measure increases when the parameters become larger, while the Divergence version is stable all the time. This suggests that RPR-NMF is not sensitive to its parameters which can be set properly without much effort.
	
	\begin{figure}[htbp]
		\begin{center}
			\includegraphics[width=0.36\textwidth]{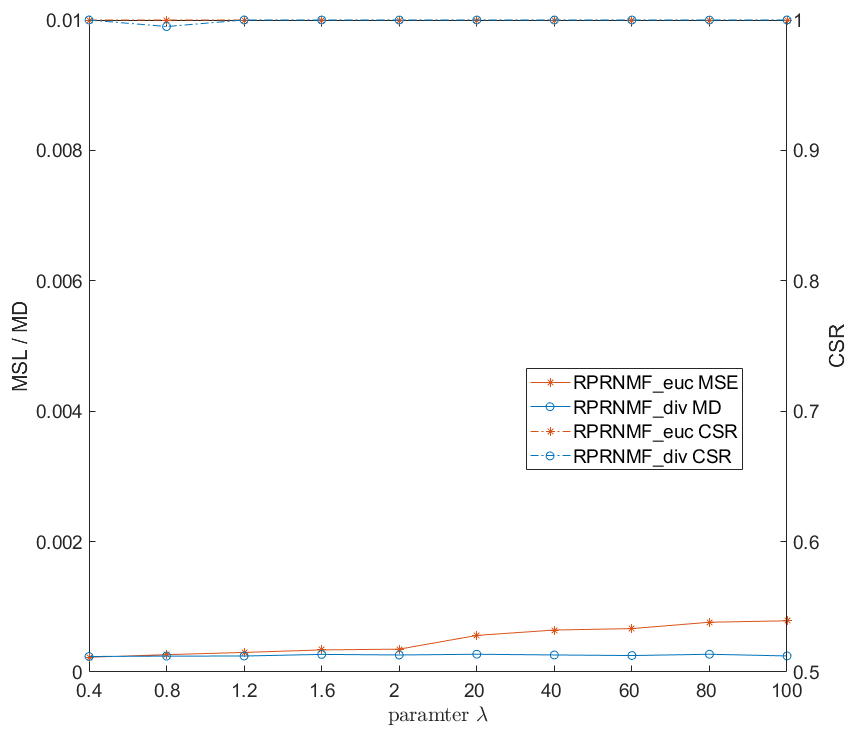}
			\caption{MSL/MD \& CSR of RPR-NMF algorithms with penalty coefficients varying from $0.4$ to $2$ and from $20$ to $100$.\label{la}}
		\end{center}
	\end{figure}
	
	\subsection{Performance Analysis in Image Clustering}
	We validate the performance of algorithms in image clustering on two public datasets with ground truth: AT\&T ORL and CMU PIE.
	
	\begin{table*}[!htbp]
		\setlength{\abovecaptionskip}{0pt} 
		\setlength{\belowcaptionskip}{0pt}
		\caption{Clustering Performance on AT\&T ORL Database (Euclidean / Divergence)}
		\begin{center}
			\label{orl}
			\begin{tabular}{|c|c||c|c|c|c||c|c|c|c|}
				\hline
				\multirow{2}{*}{$K$} & \multirow{2}{*}{$l_{\matr{H}}$} & \multicolumn{4}{c||}{MSL / MD} & \multicolumn{4}{c|}{CSR (\%)}\\
				\cline{3-10}
				& & NMF & GNMF & LCNMF & RPR-NMF & NMF & GNMF & LCNMF & RPR-NMF\\
				\hline
				5 & 10 & \textbf{339.9} / \textbf{1.604} & 397.9 / 1.980 & 355.1 / 1.677 & 340.7 / 1.606 & 90.00 / 86.00 & 98.00 / \textbf{100.0} & \textbf{100.0} / \textbf{100.0} & 98.00 / 96.00 \\
				10 & 20 & 263.2 / 1.200 & 304.5 / 1.589 & 288.8 / 1.294 & \textbf{261.7} / \textbf{1.197} & 88.00 / 87.00 & 98.00 / \textbf{100.0} & \textbf{100.0} / \textbf{100.0} & 94.00 / \textbf{100.00} \\
				20 & 40 & 253.7 / 1.155 & 267.2 / 1.610 & 284.1 / 1.293 & \textbf{252.6} / \textbf{1.147} & 92.00 / 83.50 & 91.00 / \textbf{100.0} & \textbf{100.0} / \textbf{100.0} & 91.50 / \textbf{100.00} \\
				30 & 60 & 228.6 / \textbf{1.039} & 233.4 / 1.494 & 265.5 / 1.207 & \textbf{228.1} / 1.041 & 91.00 / 89.00 & 92.33 / \textbf{100.0} & \textbf{100.0} / \textbf{100.0} & 97.00 / \textbf{100.00} \\
				40 & 80 & 211.0 / 0.960 & \textbf{209.2} / 1.400 & 247.1 / 1.116 & 210.9 / \textbf{0.957} & 91.25 / 89.25 & 93.50 / \textbf{100.0} & \textbf{100.0} / \textbf{100.0} & 97.00 / \textbf{100.00} \\
				\hline
				\multicolumn{2}{|c||}{Avg.} & 259.3 / 1.192 & 282.4 / 1.615 & 288.1 / 1.317 & \textbf{258.8} / \textbf{1.190} & 90.45 / 86.95 & 94.57 / \textbf{100.0} & \textbf{100.0} / \textbf{100.0} & 95.50 / 99.20 \\
				\hline
				\hline
				\multirow{2}{*}{$K$} & \multirow{2}{*}{$l_{\matr{H}}$} & \multicolumn{4}{c||}{ACC (\%)} & \multicolumn{4}{c|}{NMI (\%)}\\
				\cline{3-10}
				& & NMF & GNMF & LCNMF & RPR-NMF & NMF & GNMF & LCNMF & RPR-NMF\\
				\hline
				5 & 10 & 78.40 / 76.40 & 67.20 / 65.60 & 89.20 / 85.20 & \textbf{90.40} / \textbf{88.40} & 73.37 / 79.29 & 59.01 / 51.76 & 85.33 / \textbf{83.03} & \textbf{85.35} / 82.05 \\
				10 & 20 & 64.20 / 68.00 & 63.40 / 55.60 & 75.00 / 71.20 & \textbf{78.80} / \textbf{71.80} & 73.28 / 74.10 & 67.69 / 58.00 & 81.24 / 77.89 & \textbf{82.12} / \textbf{78.40} \\
				20 & 40 & 64.10 / 63.50 & 62.10 / 53.30 & 64.80 / 70.40 & \textbf{72.20} / \textbf{68.00} & 77.11 / 76.33 & 73.57 / 60.59 & 79.06 / \textbf{81.79} & \textbf{82.44} / 78.19 \\
				30 & 60 & 62.00 / 60.93 & 59.33 / 48.47 & 66.20 / 63.73 & \textbf{67.93} / \textbf{64.67} & 77.92 / 76.80 & 74.92 / 59.82 & 80.12 / 78.85 & \textbf{80.34} / \textbf{78.89} \\
				40 & 80 & 61.25 / 58.25 & 56.95 / 46.85 & 62.95 / \textbf{59.20} & \textbf{63.50} / 58.10 & 79.01 / 77.65 & 75.62 / 60.39 & 79.22 / \textbf{78.01} & \textbf{79.29} / 75.71 \\
				\hline
				\multicolumn{2}{|c||}{Avg.} & 65.99 / 65.42 & 61.80 / 53.96 & 71.63 / 69.95 & \textbf{74.57} / \textbf{70.19} & 76.14 / 76.83 & 70.16 / 58.11 & 80.99 / \textbf{79.91} & \textbf{81.91} / 78.65 \\
				\hline
			\end{tabular}
		\end{center}
	\end{table*}
	
	\begin{table*}[!htbp]
		\setlength{\abovecaptionskip}{0pt} 
		\setlength{\belowcaptionskip}{0pt}
		\caption{Clustering Performance on CMU PIE Database (Euclidean / Divergence)}
		\begin{center}
			\label{pie}
			\begin{tabular}{|c|c||c|c|c|c||c|c|c|c|}
				\hline
				\multirow{2}{*}{$K$} & \multirow{2}{*}{$l_{\matr{H}}$} & \multicolumn{4}{c||}{MSL / MD} & \multicolumn{4}{c|}{CSR (\%)}\\
				\cline{3-10}
				& & NMF & GNMF & LCNMF & RPR-NMF & NMF & GNMF & LCNMF & RPR-NMF\\
				\hline
				10 & 120 & \textbf{203.1} / \textbf{1.788} & 203.8 / 2.916 & 429.8 / 3.234 & 210.1 / 1.852 & 84.00 / 78.67 & 80.50 / 68.83 & 00.00 / 00.00 & \textbf{84.17} / \textbf{82.83} \\
				20 & 240 & 171.6 / \textbf{1.385} & 163.2 / 2.689 & 425.3 / 3.063 & \textbf{161.4} / 1.478 & 82.83 / \textbf{73.25} & 82.25 / 70.58 & 00.00 / 00.00 & \textbf{85.42} / 71.75 \\
				30 & 360 & 151.4 / \textbf{1.292} & \textbf{140.6} / 2.447 & 397.4 / 2.944 & 141.5 / 1.370 & 82.17 / 72.89 & 83.22 / 68.28 & 00.00 / 00.00 & \textbf{86.78} / \textbf{76.83} \\
				40 & 480 & 137.3 / \textbf{1.147} & \textbf{122.9} / 2.282 & 382.4 / 2.793 & 125.2 / 1.219 & 83.13 / 76.04 & 84.46 / 70.58 & 00.00 / 00.00 & \textbf{88.08} / \textbf{77.29} \\
				50 & 600 & 130.0 / \textbf{1.088} & \textbf{113.7} / 2.217 & 385.8 / 2.796 & 118.1 / 1.139 & 82.50 / \textbf{76.83} & 84.27 / 67.53 & 00.00 / 00.00 & \textbf{88.43} / 75.60 \\
				60 & 720 & 122.3 / \textbf{1.013} & \textbf{105.5} / 2.124 & 378.2 / 2.731 & 110.4 / 1.102 & 81.47 / 74.92 & 83.69 / 67.67 & 00.00 / 00.00 & \textbf{87.83} / \textbf{75.50} \\
				68 & 816 & 116.9 / \textbf{0.978} & \textbf{98.83} / 2.037 & 370.8 / 2.693 & 103.5 / 0.998 & 81.30 / \textbf{74.34} & 84.14 / 68.70 & 00.00 / 00.00 & \textbf{88.51} / 71.57 \\
				\hline
				\multicolumn{2}{|c||}{Avg.} & 147.5 / \textbf{1.242} & \textbf{135.5} / 2.387 & 395.7 / 2.893 & 138.6 / 1.308 & 82.49 / 75.28 & 83.22 / 68.88 & 00.00 / 00.00 & \textbf{87.03} / \textbf{75.91} \\
				\hline
				\hline
				\multirow{2}{*}{$K$} & \multirow{2}{*}{$l_{\matr{H}}$} & \multicolumn{4}{c||}{ACC (\%)} & \multicolumn{4}{c|}{NMI (\%)}\\
				\cline{3-10}
				& & NMF & GNMF & LCNMF & RPR-NMF & NMF & GNMF & LCNMF & RPR-NMF\\
				\hline
				10 & 120 & 64.86 / 63.19 & 51.76 / 53.71 & 60.29 / 57.52 & \textbf{66.43} / \textbf{66.76} & \textbf{67.13} / \textbf{67.10} & 57.73 / 52.70 & 55.78 / 55.98 & 67.01 / 64.38 \\
				20 & 240 & 58.00 / 62.57 & 60.79 / 51.10 & 56.88 / 57.05 & \textbf{67.36} / \textbf{66.64} & 68.54 / \textbf{72.07} & 68.63 / 59.70 & 60.88 / 61.87 & \textbf{73.61} / 71.00 \\
				30 & 360 & 62.49 / 61.89 & 58.21 / 53.51 & 57.87 / 57.16 & \textbf{66.32} / \textbf{66.73} & 71.87 / \textbf{72.60} & 70.37 / 63.14 & 64.71 / 63.66 & \textbf{75.06} / 72.55 \\
				40 & 480 & 61.73 / 59.35 & 60.23 / 54.20 & 56.25 / 54.13 & \textbf{64.51} / \textbf{66.17} & 73.15 / 73.01 & 72.66 / 63.26 & 64.40 / 65.35 & \textbf{73.94} / \textbf{74.63} \\
				50 & 600 & 60.63 / 61.62 & 58.98 / 52.15 & 56.77 / 56.17 & \textbf{64.49} / \textbf{66.24} & 73.85 / 74.40 & 73.41 / 64.53 & 66.48 / 65.06 & \textbf{75.67} / \textbf{74.96} \\
				60 & 720 & 61.36 / 60.09 & 57.06 / 53.97 & 55.23 / 53.38 & \textbf{62.31} / \textbf{66.48} & \textbf{75.47} / 74.49 & 73.35 / 66.32 & 66.68 / 66.12 & 75.33 / \textbf{75.05} \\
				68 & 816 & 57.75 / 60.78 & 58.60 / 51.51 & 52.72 / 53.79 & \textbf{62.40} / \textbf{63.10} & 73.32 / \textbf{75.25} & 74.18 / 66.31 & 66.38 / 65.65 & \textbf{75.72} / 74.21 \\
				\hline
				\multicolumn{2}{|c||}{Avg.} & 60.97 / 61.36 & 57.95 / 52.88 & 56.57 / 55.60 & \textbf{64.83} / \textbf{66.02} & 71.90 / \textbf{72.70} & 70.05 / 62.28 & 63.62 / 63.38 & \textbf{73.76} / 72.40 \\
				\hline
			\end{tabular}
		\end{center}
	\end{table*}
	
	\subsubsection{AT\&T ORL Dataset}
	The AT\&T ORL database consists of $400$ images for $40$ classes with $10$ different facial images in each class \cite{SamariaIEEE1994}. The images were taken at different times, lighting and facial expressions. The faces are in an upright position in frontal view, with a slight left-right rotation. Each image is preprocessed into a $32 \times 32$ matrix with $256$ grey levels \cite{LiuPAMI2012}. Thus the size of factorising matrix is $1,024 \times 400$.
	
	We adopt the following steps in this experiment:
	
	i. Randomly choose $K$ classes and mix up images from these classes to form the factorising matrix;
	
	ii. In the $K$ classes, randomly select $2$ images in each class and set them more similar to images chosen in other classes, which forms the list of RPR constraints for RPR-NMF; transform the constraints into a weight matrix and a label matrix; the latent dimension is set as the number of chosen classes $K$;
	
	iii. Run algorithms to obtain the right factorised matrix $\matr{H}$; utilise K-means method on $\matr{H}$ to get the clustering results;
	
	iv. Calculate the ACC and NMI for each algorithm.
	
	The number of class $K$ varies from $2$ to $10$ with step $2$. The penalty coefficients for RPR-NMF using Euclidean measure are set $20$ and for its Divergence version are set $2$. The results are showed in TABLE \ref{orl}.
	
	From the table, RPR-NMF using Euclidean measure achieves the best average ACC ($85.78\%$) while RPR-NMF using Divergence measure achieves the best average NMI ($82.33\%$). Among algorithms using Euclidean measure, RPR-NMF outperforms NMF, GNMF and LCNMF by $17.28\%$, $13.30\%$, $15.51\%$ on ACC and by $19.71\%$, $10.76\%$, $15.27\%$ on NMI respectively; as for algorithms using Divergence measure, RPR-NMF outperforms NMF, GNMF and LCNMF by $17.14\%$, $10.32\%$, $16.76\%$ on ACC and by $19.93\%$, $11.36\%$, $16.37\%$ on NMI respectively. The average improvement of RPR-NMF algorithms is $15.05\%$ on ACC and $15.57\%$ on NMI. All of the algorithms obtain very high CSR, because the images in this dataset are quite distinguishing. Besides, the there is no chain among RPR constraints, thus LCNMF can satisfy all of them by setting proper labels.
	
	\begin{table*}[!htbp]
		\setlength{\abovecaptionskip}{0pt} 
		\setlength{\belowcaptionskip}{0pt}
		\caption{Cross Validation Results on Movielens 1M Dataset (Euclidean / Divergence)}
		\begin{center}
			\label{movielens}
			\begin{tabular}{|c|c||c|c|c|c||c|c|c|c|}
				\hline
				\multirow{2}{*}{$K$} & \multirow{2}{*}{$l_{\matr{W}}$ \& $l_{\matr{H}}$} & \multicolumn{4}{c||}{MSL / MD} & \multicolumn{4}{c|}{CSR (\%)}\\
				\cline{3-10}
				& & NMF & GNMF\_euc & LCNMF & RPR-NMF & NMF & GNMF\_euc & LCNMF & RPR-NMF\\
				\hline
				20 & 300 & \textbf{0.513} / 0.083 & 0.526 & 0.532/ 0.086 & 0.514 / \textbf{0.082} & 49.07 / 49.87 & 87.37 & 50.00 / 50.00 & \textbf{96.17} / \textbf{95.77} \\
				50 & 600 & \textbf{0.373} / \textbf{0.060} & 0.409 & 0.422 / 0.068 & 0.374 / \textbf{0.060} & 49.23 / 51.03 & 89.25 & 50.00 / 50.00 & \textbf{95.13} / \textbf{98.83} \\
				100 & 900 & \textbf{0.249} / \textbf{0.039} & 0.317 & 0.334 / 0.053 & 0.253 / \textbf{0.039} & 50.17 / 51.51 & 89.72 & 49.94 / 49.94 & \textbf{93.49} / \textbf{99.39} \\
				\hline
				\multicolumn{2}{|c||}{Avg.} & \textbf{0.378} / 0.061 & 0.417 & 0.429 / 0.069 & 0.380 / \textbf{0.060} & 49.38 / 50.80 & 88.78 & 49.98 / 49.98 & \textbf{94.93} / \textbf{98.00} \\
				\hline
				\hline
				\multirow{2}{*}{$K$} & \multirow{2}{*}{$l_{\matr{W}}$ \& $l_{\matr{H}}$} & \multicolumn{4}{c||}{RMSE} & \multicolumn{4}{c|}{F1 Score (\%)}\\
				\cline{3-10}
				& & NMF & GNMF\_euc & LCNMF & RPR-NMF & NMF & GNMF\_euc & LCNMF & RPR-NMF\\
				\hline
				20 & 300 & 0.959 / 0.975 & 0.929 & 0.991 / 1.004 & \textbf{0.928} / \textbf{0.974} & 68.75 / 68.67 & 68.29 & 66.39 / 66.23 & \textbf{69.25} / \textbf{68.71} \\
				50 & 600 & 1.027 / 1.045 & \textbf{0.961} & 1.044 / 1.068 & 0.979 / \textbf{1.030} & 67.34 / \textbf{67.20} & 67.45 & 64.78 / 64.80 & \textbf{67.51} / 67.19 \\
				100 & 900 & 1.100 / \textbf{1.147} & \textbf{1.010} & 1.133 / 1.165 & 1.055 / 1.157 & 65.10 / 65.01 & 65.31 & 62.20 / 62.26 & \textbf{65.46} / \textbf{65.03} \\
				\hline
				\multicolumn{2}{|c||}{Avg.} & 1.028 / \textbf{1.055} & \textbf{0.966} & 1.056 / 1.079 & 0.987 / 1.056 & 67.06 / 66.96 & 67.02 & 64.46 / 64.43 & \textbf{67.41} / \textbf{66.98} \\
				\hline
			\end{tabular}
		\end{center}
	\end{table*}
	
	\subsubsection{CMU PIE Dataset}
	CMU PIE database was collected at Carnegie Mellon University in 2000, and it has been very influential in advancing research in face recognition across pose and illumination \cite{GrossIVC2010}. We followed the pre-processed PIE dataset used in \cite{CaiPAMI2011} which contains $2,856$ images for $68$ different people with $42$ images for each person. The images are processed into $32 \times 32$ matrices denoting the grey level of pixels. Thus the size of factorising matrix is $1,024 \times 2,856$.
	
	We adopt similar experimental steps as we did for AT\&T ORL dataset with a few changes on the extraction of RPR constraints. For this experiment, we randomly select $4$ images instead of $2$ in each cluster. Moreover, considering that the similarity can exist not only among intra-class images, but also inter-class images (e.g. the image of a dog is more similar to that of another dog, rather than a cat), we extract constraints in both ways. The number of class $K$ varies from $10$ to $60$ with step $10$. The penalty coefficients for RPR-NMF using Euclidean measure are set $20$ and for its Divergence version are set $2$. The results are presented in TABLE \ref{pie}.
	
	According to the table, we can see that GNMF using Euclidean measure and NMF using Divergence measure achieve the best average approximation ($135.5$ and $1.242$). As for the other three evaluation metrics,  RPR-NMF using Euclidean measure achieves the best average CSR ($87.03\%$) and NMI ($73.76\%$) while RPR-NMF using Divergence measure achieves the best average ACC ($66.02\%$). Notice that the CSR of LCNMF are all zeros because inter-class and intra-class constraints lead to cyclic chain constraints.
	
	Among the algorithms using Euclidean measure, RPR-NMF outperforms NMF, GNMF and LCNMF by $6.40\%$, $11.94\%$, $14.67\%$ on ACC and by $2.66\%$, $5.37\%$, $16.02\%$ on NMI respectively. For the algorithms using Divergence measure, RPR-NMF outperforms NMF, GNMF and LCNMF by $7.59\%$, $24.85\%$, $18.74\%$ on ACC. However, its performance on NMI is a bit lower than NMF, while it outperforms GNMF and LCNMF by $16.25\%$, $14.23\%$ on NMI. The average improvement of RPR-NMF algorithms is $14.03\%$ on ACC and $9.02\%$ on NMI.
	
	\subsection{Performance Analysis in Recommender Systems}
	As for the performance analysis in recommender systems, we compare our algorithm RPR-NMF with NMF, GNMF, and LCNMF on Movielens 1M dataset. For the reason that the rating matrices in recommender systems have missing values which are usually denoted by zeros, all the algorithms have to be modified with a \textit{MASK} matrix for incomplete factorising matrix as proposed in \cite{ZhangSDM2006}. To our best effort, we implemented the modified version for all algorithms except for the GNMF using Divergence measure. In the Appendix of \cite{CaiPAMI2011}, the authors only mentioned how to deal with incomplete factorising matrix for GNMF using Euclidean measure. Thus we did not compare GNMF using Divergence measure in this part.
	
	Specifically, the pairwise relationship constraints in this part are extracted from meta information and are imposed on both factorised matrices: we utilised users' gender, age and occupation as well as movies' genre to obtain RPRs for the Movielens dataset.
	
	\subsubsection{Movielens 1M Dataset}
	The Movielens 1M dataset is a well-known stable baseline dataset in recommender systems \cite{HarperTIIS2015}. It has $1,000,209$ ratings ($1$ to $5$) from $6,040$ users on $3,883$ movies. Indeed some movies have no ratings, thus after removing these movies, we derived a pre-processed $6,040 \times 3,706$ rating matrix.
	
	Three groups of cross validation experiments are conducted on this dataset: (1) $300$ constraints, $K=20$; (2) $500$ constraints and $K=50$; (3) $1000$ constraints and $K=100$. The penalty coefficients for RPR-NMF using Euclidean measure are set $200$, $10$ and $1$ respectively, and the coefficients for the Divergence version of RPR-NMF are set $0.1$, $0.01$, $0.001$ respectively. For each group, we conducted $5$ cross validation experiments and the average results are showed in TABLE \ref{movielens}.
	
	As presented in the table, NMF achieves the lowest MSL and RPR-NMF achieves the lowest MD; both of the algorithms obtain close approximation errors while the other two methods, GNMF and LCNMF, have a higher rate of errors. RPR-NMF also achieves the highest CSR on both its versions. As for the recommendation evaluation, GNMF using Euclidean measure achieves the lowest RMSE while RPR-NMF using Euclidean measure achieves the highest F1 score. It is worth noting that, in this experiment, the RPRs are randomly selected through meta information, and some of the constraints may not represent the true ratings' pattern. Thus the RPR-NMF algorithms seem not greatly outperform existing methods. However, their overall performance is still better compared to other alternatives.
	
	\section{Conclusions}
	
	In this paper, we proposed a novel matrix factorisation algorithm called RPR-NMF, to effectively utilise the relative pairwise relationship among rows or columns of factorised matrices. Both of the Euclidean and Divergence measures are used in the objective function. RPR-NMF imposes penalties for each relative pairwise relationship constraint in a form of addition of natural exponential functions for Euclidean measure and hinge loss for Divergence measure. Complete and sufficient proofs of convergence are also provided to ensure that RPR-NMF conforms to the ``multiplicative update rules''. Numerical analysis shows that the proposed algorithm achieves superior performance on both the overall loss and the accuracy of satisfied constraints compared with the other algorithms. Experiments on synthetic and real datasets for image clustering and recommender systems demonstrate the effectiveness of RPR-NMF algorithms which outperform baseline methods on several evaluation criteria.

	\bibliographystyle{IEEEtran}
	\bibliography{arxiv}
	
	%
	
	\begin{IEEEbiography}[{\includegraphics[width=1in,height=1.25in,clip,keepaspectratio]{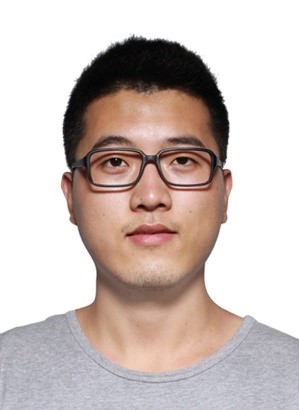}}]{Shuai Jiang}
		received the bachelor’s degree in computer science and technology from Beijing Institute of Technology, Beijing, China, in 2013. Currently he is working towards the dual doctoral degree in both Beijing Institute of Technology and University of Technology Sydney. His main interests include machine learning, optimisation and data analytics.
	\end{IEEEbiography}
	
	\begin{IEEEbiography}[{\includegraphics[width=1in,height=1.25in,clip,keepaspectratio]{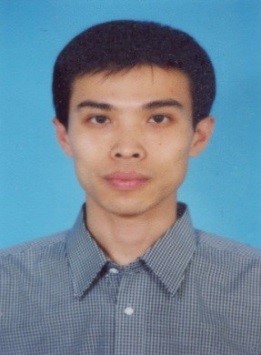}}]{Kan Li}
		is currently a Professor in the School of Computer at Beijing Institute of Technology. He has published over 50 technical papers in peer-reviewed journals and conference proceedings. His research interests include machine learning and pattern recognition.
	\end{IEEEbiography}
	
	
	\begin{IEEEbiography}[{\includegraphics[width=1in,height=1.25in,clip,keepaspectratio]{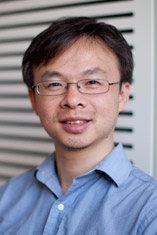}}]{Richard Yida Xu}
		received the B.Eng. degree in computer engineering from the University of New South Wales, Sydney, NSW, Australia, in 2001, and the Ph.D. degree in computer sciences from the University of Technology at Sydney (UTS), Sydney, NSW, Australia, in 2006. He is currently an Associate Professor of School of Electrical and Data Engineering, UTS. His current research interests include machine learning, deep learning, data analytics and computer vision.
	\end{IEEEbiography}
	
	
	

\end{document}